\documentclass{article}

\PassOptionsToPackage{numbers}{natbib}

\usepackage[main, final]{neurips_2026}

\usepackage[utf8]{inputenc} %
\usepackage[T1]{fontenc}    %
\usepackage{hyperref}       %
\usepackage{url}            %
\usepackage{booktabs}       %
\usepackage{amsfonts}       %
\usepackage{nicefrac}       %
\usepackage{microtype}      %
\usepackage{xcolor}         %

\title{Dyna-Style Safety Augmented Reinforcement Learning: Staying Safe in the Face of Uncertainty}

\usepackage{wrapfig}
\usepackage{amssymb}
\usepackage{algorithm}
\usepackage{subcaption}
\usepackage{algorithmic}
\usepackage{array}
\usepackage{adjustbox}
\usepackage{amsmath,bm}
\usepackage{mathtools}
\usepackage{amsthm}
\usepackage{thmtools}
\usepackage{xcolor}
\usepackage{tikz}
\usetikzlibrary{decorations.pathreplacing}
\usetikzlibrary{patterns}
\usetikzlibrary{calc}

\newcommand{\accurset}[1]{\mathcal{E}_{#1}}

\newcommand{\singleStepThreshold}{\lambda_1}
\newcommand{\multiStepThreshold}{\lambda_2}
\newcommand{\informationLoss}{\mathbb{H}}

\def\Figref#1{Figure~\ref{#1}}

\def\Algref#1{Algorithm~\ref{#1}}
\def\Secref#1{Section~\ref{#1}}

\newcommand{\R}{\mathbb{R}}

\def\1{\bm{1}}

\DeclareMathOperator*{\argmax}{arg\,max}

\newcommand{\midd}{\,\middle|\,}

\newcommand{\lb}{\left[}
\newcommand{\rb}{\right]}

\newcommand{\ALGCOMMENT}[1]{\hfill{\itshape\footnotesize #1}}

\newcommand*{\dynasaur}{Dyna-SAuR}
\newcommand*{\dynasaurFull}{Dyna-style Safety Augmented Reinforcement Learning}

\definecolor{greenSafe}{RGB}{65,117,5}
\definecolor{redPerf}{RGB}{204,7,30
}

\newcommand{\mdp}{\mathcal{M}}
\newcommand{\statesp}{\mathcal{S}}
\newcommand{\actsp}{\mathcal{A}}
\newcommand{\rewfcn}{r}
\newcommand{\dynfcn}{p}

\newcommand{\initdist}{\rho_0}
\newcommand{\discount}{\gamma}
\newcommand{\state}{s}
\newcommand{\rvstate}{S}
\newcommand{\modrvstate}{\hat{S}}
\newcommand{\act}{a}
\newcommand{\rvact}{A}
\newcommand{\rew}{r}
\newcommand{\rvrew}{R}

\newcommand{\moddyn}{\hat{p}}
\newcommand{\modparams}{\omega}
\newcommand{\modstate}{\hat{s}}
\newcommand{\rvmodstate}{\hat{S}}

\newcommand{\pol}{\pi}
\newcommand{\polparams}{\psi}
\newcommand{\polset}{\Pi}
\newcommand{\filterpolset}{\mathrm{M}}
\newcommand{\qfcn}{Q}
\newcommand{\qfcnparams}{\vartheta}

\newcommand{\filterpol}{\mu}
\newcommand{\filterpolparams}{\phi}
\newcommand{\filterqfcn}{Q^{\mathrm{SF}}}
\newcommand{\filterqfcnparams}{\theta}

\newcommand{\algiter}{j}
\newcommand{\maxalgiter}{J}

\newcommand{\unidist}{\mathbf{U}}

\newcommand{\pinknoiseproc}{\mathrm{PN}}

\newcommand{\behpol}{\beta}

\newcommand{\noise}{\epsilon}

\newcommand{\buffer}{\mathcal{D}}

\newcommand{\expscale}{\sigma}

\newcommand{\entropy}{\mathbb{H}}
\newcommand{\enttresh}{\lambda}

\newcommand{\viaact}{a^{\mathrm{V}}}

\newcommand{\polyak}{\bar{\tau}}
\newcommand{\trajfail}{\tau_\mathrm{F}}
\newcommand{\trajfailset}{\mathcal{T}_\mathrm{F}}
\newcommand{\trajret}{\tau_\mathrm{G}}
\newcommand{\trajretset}{\mathcal{T}_\mathrm{G}}

\newcommand{\na}{n_{\actsp}}
\newcommand{\ns}{n_{\statesp}}

\newcommand{\safepolset}{\Pi_{\mathrm{V}}}
\newcommand{\filter}{\kappa}
\newcommand{\failureset}{\mathcal{S}_{\mathrm{F}}}

\newcommand{\safeset}{\mathcal{S}_{\mathrm{S}}}

\newcommand{\safesetSFi}[1]{\safeset(\accurset{#1})}

\newcommand{\modelViableSeti}[1]{\viabilitykernel(\accurset{#1})}
\newcommand{\modelSafePolsSet}[1]{\safepolset(\accurset{#1})}

\newcommand{\viabilitykernel}{\mathcal{S}_{\mathrm{V}}}
\newcommand{\unviabilitykernel}{\mathcal{S}_{\mathrm{U}}}

\newcommand{\mdpsf}{\mathcal{M}^{\mathrm{SF}}}
\newcommand{\actspsf}{\mathcal{U}}
\newcommand{\dynfcnsf}{p^{\mathrm{SF}}}
\newcommand{\rewfcnsf}{r^{\mathrm{SF}}}
\newcommand{\initdistsf}{\rho^{\mathrm{SF}}_0}
\newcommand{\initdistsfeval}{\rho^{\mathrm{EVAL}}_0}

\newcommand{\asf}{u}
\newcommand{\rewsf}{r^\mathrm{SF}}

\newcommand{\discountsf}{\gamma_\mathrm{SF}}

\newcommand{\filteredAction}{a^{\mathrm{V}}}

\newcommand{\hyperplaneOffset}{b}
\newcommand{\hyperplaneVector}{w}
\newcommand{\hyperplaneTransformationFunction}{h}

\newcommand{\safetyFilterModelBasedRollouts}{\mathrm{RO}^{\mathrm{SF}}}

\makeatletter\@ifundefined{iflatexml}{\newif\iflatexml}{}\makeatother
\usepackage{multirow}
\usepackage{comment}
\usepackage[most]{tcolorbox}
\usepackage{paralist}
\usepackage{adjustbox}

\definecolor{steelblue}{RGB}{70,130,180}
\definecolor{steelblue!5}{RGB}{230,238,245}  %
\definecolor{steelblue!75!black}{RGB}{37,77,108}  %
\definecolor{dynared}{RGB}{214, 39, 40}
\definecolor{berndcol}{RGB}{20, 105, 176}
\definecolor{arturcol}{RGB}{34, 139, 34}
\theoremstyle{plain}
\newtheorem{theorem}{Theorem}[section]

\newtheorem{lemma}[theorem]{Lemma}

\theoremstyle{definition}
\newtheorem{definition}[theorem]{Definition}

\theoremstyle{remark}

\iflatexml
\author{%
  Artur Eisele\textsuperscript{$\dagger$,1}
  \and Bernd Frauenknecht\textsuperscript{$\dagger$}
  \and Friedrich Solowjow
  \and Sebastian Trimpe \and
  Institute for Data Science in Mechanical Engineering, RWTH Aachen University
}
\else
\author{%
   Artur Eisele\textsuperscript{$\dagger$,1}
  \quad Bernd Frauenknecht\textsuperscript{$\dagger$}
  \quad Friedrich Solowjow
  \quad Sebastian Trimpe \AND
  \normalfont Institute for Data Science in Mechanical Engineering, RWTH Aachen University
}
\fi

\begin{document}

\maketitle

\begingroup
\renewcommand{\thefootnote}{}
\footnotetext{\textsuperscript{\textdagger}Equal contribution. \quad \textsuperscript{1}Corresponding author: \texttt{\{first name\}.\{surname\} [at] dsme.rwth-aachen.de}}
\endgroup
\addtocounter{footnote}{1}
\begin{abstract}
Safety is a key problem in reinforcement learning, especially during exploration. While safety filters are promising to address this issue, they are ill-suited for high-dimensional systems with unknown dynamics. We propose \emph{\dynasaurFull{} (\dynasaur{})}, a novel algorithm that learns a control policy as well as a safety filter requiring minimal domain knowledge by leveraging a learned uncertainty-aware dynamics model. A novel filter learning formulation requires the filter to avoid failures and areas of high model uncertainty. As joint deployment of the control policy and filter allows to safely collect data and improve the dynamics model, certain areas grow and filter conservatism decreases over time.
\dynasaur{} reduces failures compared to established safe RL baselines by two orders of magnitude on goal-reaching CartPole and MuJoCo Walker.
\end{abstract}

\section{Introduction}

\begin{wrapfigure}{r}{0.45\textwidth}
    \centering
    \vspace{-17pt}
    \iflatexml\includegraphics[width=0.45\textwidth]{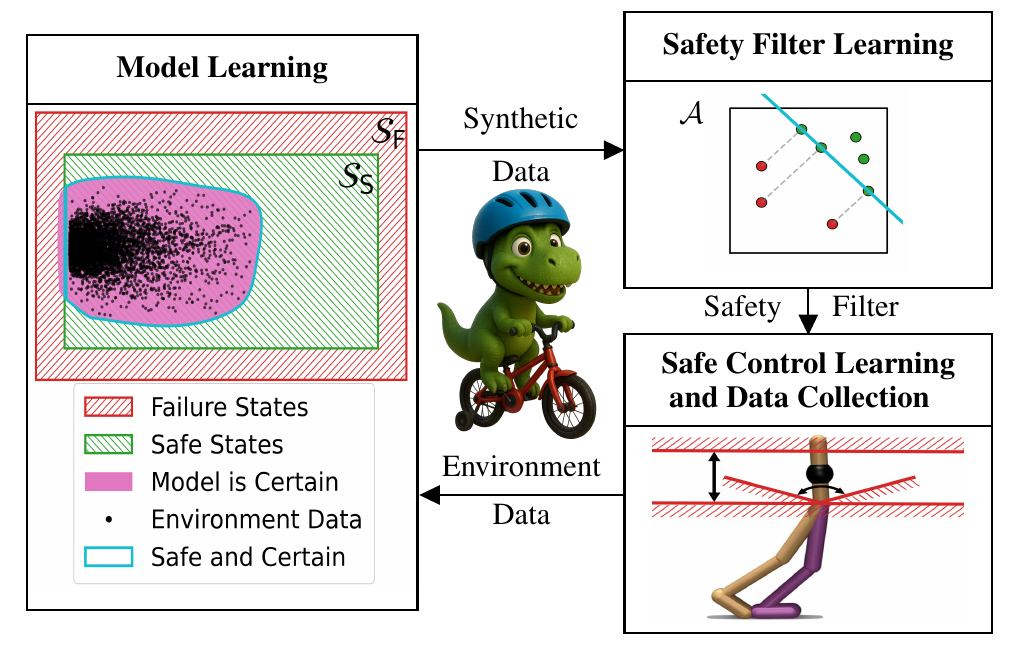}\else\include{Tex/Tikz/Introduction/figure1}\fi
    \vspace{-20pt}
    \caption{\dynasaur{} mechanism. \emph{ An uncertainty-aware dynamics model is used to train a safety filter that avoids both failures and uncertain regions of the model. The filter is leveraged to safely learn a control policy in the environment. The collected interaction data is used to improve the dynamics model, which expands the certain area and reduces filter conservatism in the next iteration.}}
    \vspace{-30pt}
    \label{fig:Figure1} 
\end{wrapfigure}

Safety during training is a major challenge in reinforcement learning (RL). Safe RL methods typically treat safety as a soft constraint \cite{achiam2017constrained, actsafe_ICLR_2025_similarToIsol} or use safety filters from control theory \cite{bansal2017hamilton, ames2019control, wabersich2023data}. The former ensures safety only in expectation, while the latter requires extensive domain knowledge and struggles with high-dimensional problems.
While recent results suggest safety filters can be obtained in a data driven way using RL \cite{black-box-lavankul}, limitations such as access to a simulator and scaling of the method remain.

We introduce \emph{\dynasaurFull{}} (\dynasaur), a model-based reinforcement learning (MBRL) method that concurrently learns a control policy and a safety filter leveraging an uncertainty-aware dynamics model. \dynasaur{} enables safe exploration, solely relying on a notion of failure and initial environment data.

\paragraph{Main Idea}
Evaluating safety involves identifying if subsequent failures occur, which can be assessed leveraging a learned dynamics model in synthetic rollouts without risk. However, the model is only reliable where it accurately models the environment, as illustrated in Figure \ref{fig:Figure1}. We train a safety filter to avoid both unsafe and uncertain regions. Deploying this filter enables safe interaction of a control policy with the environment within the model's certain area. The interaction data expands the certain area. This yields a less restrictive safety filter in the next iteration, granting more freedom to control.

To build such an architecture, we combine the notions of safety and certainty within the learned model, propose a compact RL formulation of the safety filter learning problem, and provide practical guidelines to scale the method up. In empirical evaluations on goal-reaching CartPole and MuJoCo Walker, \dynasaur{} matches or exceeds the performance of established safe RL baselines, while reducing failures during training by at least two orders of magnitude.

\section{Background}\label{sec:background}
First, we introduce the necessary background on MBRL and safe learning to then specify the technical challenges we address for learning the safety filter in a model-based fashion leveraging RL.

\subsection{Model-based Reinforcement Learning}
\label{subsec:mbrl}
We consider sequential decision-making in an environment modelled as a discrete-time Markov decision process (MDP) $\mdp = \{ \statesp, \actsp, \dynfcn, \rewfcn, \initdist, \discount \}$, with states $\rvstate_t \in \statesp \subseteq \R^{\ns}$ and actions $\rvact_t \in \actsp = [-1, 1]^{\na}$. The latter can be readily achieved through normalization. Rewards $\rvrew_t \in \mathbb{R}$ are generated by the reward function $\rewfcn: \statesp \times \actsp \times \statesp \rightarrow \mathbb{R}$, while the environment dynamics $\dynfcn: \statesp \times \actsp \rightarrow \Delta(\statesp)$ propagate the system state, with $\Delta$ indicating a stochastic mapping. Starting from an initial state $\rvstate_0 \sim \initdist$, the goal of RL is to find an optimal policy $\pol^* : \statesp \rightarrow \actsp$ within the set of policies $\polset$ that maximizes the expected return, i.e., the expected sum of rewards discounted by the factor $0 \leq \discount < 1$,
    \begin{equation}
        \pol^* = \arg \max_{\pol \in \polset} \mathbb{E}_{\pol} \left[ \sum_{t=0}^{\infty} \discount^{t} \rvrew_{t+1} \right].
        \label{eq:rl_obj}
    \end{equation}
Where needed, we use $\rvstate_t$, $\rvact_t$, and $\rvrew_t$ to refer to random variables and distinguish them from their realizations $\state_t$, $\act_t$, and $\rew_t$. Further, the state-action occupancy of a policy $\pol$
\begin{equation}
    \rho_{\pol}(\state, \act) := \lim_{T \rightarrow \infty} \frac{1}{T} \sum_{t=0}^{T-1} \mathbb{P}_\pol \lb \rvstate_t = \state, \rvact_t = \act \rb
\end{equation}
describes the probability of being in a particular state-action pair under this policy.

To reduce costly and dangerous data collection in the environment, model-based RL trains dynamics models $\moddyn(\cdot | \rvstate_t, \rvact_t)$ to approximate the environment dynamics~\cite{Deisenroth2011Jun, Hafner2023Jan}. In Dyna-style architectures~\cite{Sutton1991Jul, janner_when_2019}, $\moddyn$ is used to simulate interactions to train model-free RL algorithms. However, $\moddyn$ is typically unreliable outside the training data distribution. Therefore, models such as the probabilistic ensemble (PE)~\cite{Lakshminarayanan2017Dec} capture epistemic uncertainty, where ensemble disagreement indicates prediction inaccuracy. This uncertainty is used to define a sufficiently certain set $\accurset{} \subseteq (\statesp \times \actsp)$ where the environment is modeled accurately~\cite{macura}. Following~\cite{Frauenknecht2025Infoprop}, we leverage the sufficiently certain set
\begin{equation}
\accurset{} := \{(\state_t, \act_t) \in \statesp \times \actsp \mid \entropy(\modrvstate_{t+1}) \leq \singleStepThreshold\},
\label{eq:accur_set}
\end{equation}
where $\entropy(\modrvstate_{t+1})$ is the entropy of the model prediction that remains due to epistemic uncertainty. Low entropy corresponds to a good approximation of the environment dynamics $\dynfcn$, thus predictions are considered trustworthy below a threshold $\singleStepThreshold$. We refer to~\cite{Frauenknecht2025Infoprop,frauenknecht2026modelswrongknowinguseful} for a detailed presentation of the Infoprop state $\modrvstate_{t+1}$ and to Appendix~\ref{app:modelLearning} for details on model learning and rollout generation.

\subsection{Safe Learning}
\label{subsec:safe_learning}
We focus on training-time safety~\cite{safeLearningInRobotics}, minimizing failures during learning rather than enforcing safety only at deployment of the converged agent.
Failures are defined as entering a set of failure states $\failureset \subseteq \statesp$, such that the complement $\safeset \coloneq \statesp \setminus \failureset$ defines the set of safe states\footnote{We use terminology consistent with viability theory, differing from some works in safe RL, as discussed in Appendix~\ref{subsec:safe_and_viable}.}. However, some safe states inevitably lead to failure. Therefore, we aim to stay within a finite horizon\footnote{Since we consider unbounded process noise, we cannot make infinite horizon statements~\cite{UnboundedSupportMeansNoSafety}.}  viability kernel $\viabilitykernel \subseteq \safeset$ that we define inspired by probabilistic viability theory~\cite{viability_piere2025, Aubin2011_Viability}.
\begin{definition}(\textbf{Finite-Horizon Viability}) 
A state within the finite horizon viability kernel
\begin{equation}
\begin{aligned}
\viabilitykernel  \coloneq \big\{ \state \in \statesp \mid \exists \pol \in \polset: \mathbb{P}_{\pol} \lb \forall t \leq T,\,
\rvstate_t \in \safeset \mid \rvstate_0 = \state \rb \geq 1-\delta \big\}
\end{aligned}
\end{equation}
is viable. That is, there exists a policy such that the system remains within the set of safe states with high probability $1-\delta$ for a considered finite time horizon $T \in \mathbb{N}$ when starting from $\state$.
\end{definition}
The unviability kernel $\mathcal{S}_\mathrm{U} := \statesp \setminus \viabilitykernel$ includes all states for which failure within horizon $T$ is unavoidable with high probability, irrespective of the applied controls.

\begin{definition}
    (\textbf{Viable Action}) For $\state_t \in \viabilitykernel$, a viable action $\viaact_t$ preserves viability of the next state.
\end{definition}
\begin{definition}\label{def:viable_policy}
(\textbf{Viable Policy}) The set of viable policies $\safepolset$ consists of all policies that select only viable actions and thus preserve viability with high probability $1-\delta$ for a horizon $T$
\begin{equation}
\begin{aligned}
\safepolset \coloneq \big\{\pi \in \polset \mid \forall \state \in \viabilitykernel: 
 \mathbb{P}_{\pol} \lb \forall t \leq T,\,
\rvstate_t \in \safeset \mid \rvstate_0 = \state \rb \geq 1-\delta \big\}.
\end{aligned}
\end{equation}
\end{definition}
A safety filter $\filter : \polset \to \safepolset$ avoids failures by projecting the set of policies into the set of viable policies. This is achieved by mapping actions of the policy in $\polset$ to the closest viable action.
For control-affine dynamics $\dynfcn$, the discriminator between viable and unviable actions is a state-dependent hyperplane $\hyperplaneVector_t^\top \, \act \geq \hyperplaneOffset_t$ with a normal vector $\hyperplaneVector_t\in \R^{\na}$ and an offset $\hyperplaneOffset_t\in \R$ \cite{black-box-lavankul}. Thus,
\begin{equation}
\begin{aligned}
   \viaact_t =  \filter^\pol(\state_t) := \arg \min_{\act \in \actsp} \| \act - \pol(\state_t) \|_2
\text{ s.t. } \hyperplaneVector_t^\top  \act \geq \hyperplaneOffset_t
\end{aligned}
\end{equation}
describes the safety filter problem for a control policy $\pol \in \polset$.
However, obtaining these state-dependent hyperplanes using classical methods~\cite{ames2019control, bansal2017hamilton} requires substantial domain knowledge and scales poorly with the system dimension~\cite{chen2018HamiltonJacobiReachability, wabersich2023data}. Thus,~\citet{black-box-lavankul} formulate the idea of training a filter policy $\filterpol(\state_t)$ that parametrizes the hyperplane via a function $\hyperplaneTransformationFunction$, resulting in the filter
\begin{equation}
\begin{aligned}
\label{eq:learned_filter}
   \act_t =  \filter^\pol_\filterpol(\state_t) := \arg \min_{\act \in \actsp} \| \act - \pol(\state_t) \|_2
   \text{ s.t. } \hyperplaneVector_t^\top  \act \geq \hyperplaneOffset_t \text{ with } (\hyperplaneVector_t, \hyperplaneOffset_t) = \hyperplaneTransformationFunction(\filterpol(\state_t)).
\end{aligned}
\end{equation}
The learned filter $\filter^\pol_\filterpol: \polset \rightarrow \mathrm{Img}(\filter^\pol_\filterpol)$ with image $\mathrm{Img}(\filter^\pol_\filterpol) \subseteq \polset$ does not necessarily project into $\safepolset$ for all $\filterpol$ in the set of filter policies $ \filterpol \in \filterpolset$. Instead, a suitable $\mu \in \filterpolset$ yielding a valid safety filter with $\mathrm{Img}(\filter^\pol_\filterpol) \subseteq \safepolset$ needs to be found by carefully designing the learning problem.

\section{Problem Statement}
\label{sec:problem_statement}
Following our main idea (cf. Figure \ref{fig:Figure1}),
\dynasaur{} requires finding the filter policy $\filterpol$ and the control policy $\pol$. Both of them leverage a learned uncertainty-aware dynamics model $\moddyn$ with a sufficiently certain subset $\accurset{}$ defined in \eqref{eq:accur_set} and the set of viable policies $\modelSafePolsSet{} \subseteq \safepolset$ that additionally remain within $\accurset{}$. Before defining $\modelSafePolsSet{}$, we formulate the \textbf{Filter Policy Learning Problem}\footnote{See Appendix \ref{subsec:clari_filter_prob} for a more detailed discussion of the notation of equation$\eqref{eq:filter_learning_problem}$.}
\begin{equation}
\begin{aligned}
\label{eq:filter_learning_problem}
\filterpol^* =& \argmax_{\filterpol \in \filterpolset} \{ \mathrm{Img}(\filter_{\filterpol}^\pol) \subseteq \modelSafePolsSet{} \}.
\end{aligned}
\end{equation}
That is, finding a filter policy $\filterpol^*$ that projects into the largest possible subset of $\modelSafePolsSet{}$. Projecting into $\modelSafePolsSet{}$ makes $\filter^\pol_\filterpol$ a valid safety filter, while the largest subset yields the least restrictive filter.

Given a filter policy $\filterpol$, we aim to find the control policy $\pol^*$ that maximizes the expected return of the filtered policy $\filter^\pol_\filterpol$, resulting in the \textbf{Control Policy Learning Problem}
\begin{equation}
\label{eq:filtered_cont_learning_prob}
            \pol^* = \argmax_{\pol \in \polset}\mathbb{E}_{\filter^\pol_\filterpol} \lb \sum_{t=0}^{\infty} \discount^t \rvrew_{t+1} \rb.
\end{equation}
To address the filter policy learning problem \eqref{eq:filter_learning_problem} via MBRL, we answer the following questions:
\begin{compactitem}
\item[\emph{(i)}] \emph{How can the concept of remaining certain be integrated into the viability kernel? }
\item[\emph{(ii)}] \emph{What is a unique representation of the hyperplane for an efficient action space design?}
\item[\emph{(iii)}] \emph{What is an informative data distribution to learn $\filterpol$ in the face of a finite sampling budget?}
\end{compactitem}
Subsequently, we combine recent results in model-based data generation \cite{Frauenknecht2025Infoprop}, actor-critic learning \cite{fasttd3}, and filtered policy learning~\cite{safetyFilterWhileTraining, markgraf2025SafeReinforcementLearning}
to address control policy learning \eqref{eq:filtered_cont_learning_prob}.
We limit the input to the learning algorithm to initial data $\buffer_0^{\moddyn, \dynfcn}$ from the environment and the failure set $\failureset$.

The remainder of this paper focuses on the key technical contributions for addressing \eqref{eq:filter_learning_problem}. In particular, Section \ref{sec:model_viability} focuses on question \emph{(i)}, while Section \ref{sec:method} discusses questions \emph{(ii)} and \emph{(iii)}. Section \ref{sec:method} further presents a high level overview of the full algorithm, supporting intuition for \dynasaur{}. Technical details, especially for addressing $\eqref{eq:filtered_cont_learning_prob}$, are discussed in Appendix \ref{app:pseudocode}.

\section{Viability in the Uncertainty-Aware Model}
\label{sec:model_viability}
To prevent failures in the environment, $\filterpol$ is trained exclusively in the dynamics model. Thus, we need to address question \emph{(i)} and extend the viability concepts introduced in Section \ref{subsec:safe_learning} to account not only for safety but also certainty under the model dynamics.
 Based on the sufficiently certain subset $\accurset{}$, the certain safe set $\safesetSFi{}$ comprises state-action pairs where the model is certain and the state is safe
\begin{equation}
\safesetSFi{}  :=\big\{ (\state_t, \act_t) \in \statesp \times \actsp \mid \entropy(\modrvstate_{t+1}) \leq \singleStepThreshold, \state_t\in \safeset\big\}.
\end{equation}
Consequently, the finite-horizon certain viability kernel
\begin{equation}
\begin{aligned}
     \modelViableSeti{} := \Big\{ \state \in \statesp \mid \exists \pol \in \polset:\mathbb{P}_{\pol} \lb \forall t \leq T,
(\modrvstate_t,\rvact_t) \in \safesetSFi{} \mid \rvstate_0 = \state \rb \geq 1-\delta \Big\}
\label{eq:viable_kernel_certain}
\end{aligned}
\end{equation}
comprises all states from which a policy exists that keeps the model in the certain safe set for a given time horizon with high probability.
Finally, the set of certain viable policies 
\begin{equation}
\begin{aligned}
\modelSafePolsSet{} \coloneq \Big\{\pi \in \polset \mid \forall \state \in \modelViableSeti{}: 
 \mathbb{P}_{\pol} \lb \forall t \leq T,
(\modrvstate_t,\rvact_t) \in \safesetSFi{} \mid \rvstate_0 = \state \rb \geq 1-\delta \Big\}
\label{eq:viable_pol_in_model}
\end{aligned}
\end{equation}
allows the system to stay within the certain viability kernel with high probability.
Thus, using the model $\moddyn$ for filter training restricts~\eqref{eq:filter_learning_problem} to the subspace $\modelSafePolsSet{} \subseteq \safepolset$, which becomes less restrictive as $\accurset{}$ grows. Safety transfers to the environment only insofar as~$\hat{p}$ approximates~$p$ well everywhere within~$\mathcal{E}$.
We evaluate this approximation empirically in Appendix~\ref{subsec:modelAccuracy}.

\section{\dynasaur{}: \dynasaurFull{ }}
\label{sec:method}
\setlength{\intextsep}{2pt}
\iflatexml\begin{algorithm}\else\begin{wrapfloat}{algorithm}[16]{r}{0.50\textwidth}\fi
\caption{\dynasaur}
\label{alg:dynasaur}

\begin{algorithmic}[1]

\INPUT $\buffer^{\moddyn, \dynfcn}_0$, randomly initialized $\moddyn_0$, $\pol_0$ and $\filterpol_0$
\FOR{$\algiter \in \{1, \dots, \maxalgiter \}$ \dynasaur~iterations}
\STATE Train dynamics model $\moddyn_{\algiter}$ using data from $\buffer^{\moddyn, \dynfcn}_{\algiter-1}$ (Algorithm \ref{alg:model_learning})
\WHILE{Filter evaluation not successful}
\STATE Train filter policy $\filterpol_{\algiter}$ using $\moddyn_{\algiter}$ and $\pol_{\algiter-1}$\\ (Algorithm \ref{alg:filter_learning})
\STATE Evaluate filter policy $\filterpol_{\algiter}$ (Algorithm \ref{alg:filter_eval})
\ENDWHILE
\STATE Train control policy $\pol_\algiter$ using $\moddyn_{\algiter}$ and $\filterpol_{\algiter}$ and collect env. data $\buffer^{\moddyn, \dynfcn}_{\algiter}$ (Algorithm \ref{alg:performance_learning})
\ENDFOR
\OUTPUT $\moddyn_{\maxalgiter}, \filterpol_{\maxalgiter}, \pol_{\maxalgiter}$
\end{algorithmic}
\iflatexml\end{algorithm}\else\end{wrapfloat}\fi

In the following, we introduce the \dynasaur{} architecture, which at its core addresses the two learning problems \eqref{eq:filter_learning_problem} and \eqref{eq:filtered_cont_learning_prob}. In the following, we focus on the filter policy learning problem \eqref{eq:filter_learning_problem} as addressing the corresponding technical challenges constitutes our main technical contribution. The solution to \eqref{eq:filtered_cont_learning_prob} then builds on recent results in RL as discussed in Appendix \ref{app:pseudocode}.

Algorithm \ref{alg:dynasaur} outlines the iterative training process.\footnote{We use iteration index $\algiter$ as a subscript for learnable objects, denoting parameters at convergence after the $\algiter^{\mathrm{th}}$ iteration.} First, the dynamics model $\moddyn$ is trained on initial environment data $\buffer^{\moddyn, \dynfcn}_0$ (line 2). Second, we learn a filter policy $\filterpol$ exclusively from model-based rollouts (line 4) until it passes an evaluation indicating successful safety filtering (line 5). Third, the control policy $\pol$, filtered with $\filterpol$,  is trained using both model data and environment interactions (line 7). Observed environment transitions are added to $\buffer^{\moddyn, \dynfcn}_{\algiter}$ (line 7). Retraining $\moddyn$ on this expanded data (line 2) in iteration $\algiter+1$ enlarges $\accurset{}$, enabling a less restrictive $\filterpol$. At each iteration $\algiter$, both policies are retrained from scratch, to avoid local optima \cite{nikishin2022primacy}.
Next, we discuss learning the filter policy $\filterpol$ and refer to Appendix~\ref{app:pseudocode} for a detailed description of the full algorithm.

\subsection{Safety Filter MDP}
\label{subsec:filter_mdp}

\begin{figure}
    \centering
    \begin{minipage}[b]{0.45\textwidth}
        \centering
        \iflatexml\includegraphics[width=0.45\textwidth]{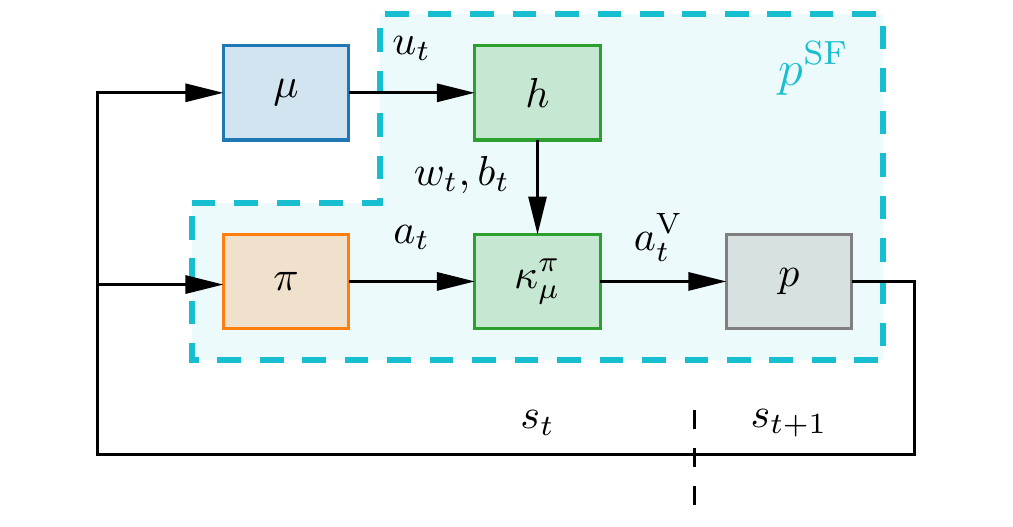}\else\tikzset{every picture/.style={line width=0.75pt}} %
\resizebox{0.85\linewidth}{!}{
\begin{tikzpicture}[x=0.75pt,y=0.75pt,yscale=-1,xscale=1]
\draw [color={rgb, 255:red, 23; green, 190; blue, 207 }  ,draw opacity=1 ][line width=1.5]  [dash pattern={on 5.63pt off 4.5pt}]  (40,70) -- (100,70) -- (100,10) -- (260,10) -- (260,120) -- (40,120) -- cycle ;

\draw  [draw opacity=0][fill={rgb, 255:red, 23; green, 190; blue, 207 }  ,fill opacity=0.08 ] (40,70) -- (100,70) -- (100,120) -- (40,120) -- cycle ;
\draw  [draw opacity=0][fill={rgb, 255:red, 23; green, 190; blue, 207 }  ,fill opacity=0.08 ][dash pattern={on 5.63pt off 4.5pt}][line width=1.5]  (100,10) -- (260,10) -- (260,120) -- (100,120) -- cycle ;
\draw  [color={rgb, 255:red, 31; green, 119; blue, 180 }  ,draw opacity=1 ][fill={rgb, 255:red, 31; green, 119; blue, 180 }  ,fill opacity=0.2 ] (50,20) -- (90,20) -- (90,50) -- (50,50) -- cycle ;
\draw  [color={rgb, 255:red, 255; green, 127; blue, 14 }  ,draw opacity=1 ][fill={rgb, 255:red, 255; green, 127; blue, 14 }  ,fill opacity=0.2 ] (50,80) -- (90,80) -- (90,110) -- (50,110) -- cycle ;
\draw  [color={rgb, 255:red, 44; green, 160; blue, 44 }  ,draw opacity=1 ][fill={rgb, 255:red, 44; green, 160; blue, 44 }  ,fill opacity=0.2 ] (130,20) -- (170,20) -- (170,50) -- (130,50) -- cycle ;
\draw  [color={rgb, 255:red, 44; green, 160; blue, 44 }  ,draw opacity=1 ][fill={rgb, 255:red, 44; green, 160; blue, 44 }  ,fill opacity=0.2 ] (130,80) -- (170,80) -- (170,110) -- (130,110) -- cycle ;
\draw  [color={rgb, 255:red, 127; green, 127; blue, 127 }  ,draw opacity=1 ][fill={rgb, 255:red, 127; green, 127; blue, 127 }  ,fill opacity=0.2 ] (210,80) -- (250,80) -- (250,110) -- (210,110) -- cycle ;
\draw    (90,35) -- (128,35) ;
\draw [shift={(130,35)}, rotate = 180] [fill={rgb, 255:red, 0; green, 0; blue, 0 }  ][line width=0.08]  [draw opacity=0] (12,-3) -- (0,0) -- (12,3) -- cycle    ;
\draw    (90,95) -- (128,95) ;
\draw [shift={(130,95)}, rotate = 180] [fill={rgb, 255:red, 0; green, 0; blue, 0 }  ][line width=0.08]  [draw opacity=0] (12,-3) -- (0,0) -- (12,3) -- cycle    ;
\draw    (170,95) -- (208,95) ;
\draw [shift={(210,95)}, rotate = 180] [fill={rgb, 255:red, 0; green, 0; blue, 0 }  ][line width=0.08]  [draw opacity=0] (12,-3) -- (0,0) -- (12,3) -- cycle    ;
\draw    (250,95) -- (270,95) -- (270,150) -- (10,150) -- (10,96) -- (48,96) ;
\draw [shift={(50,96)}, rotate = 180] [fill={rgb, 255:red, 0; green, 0; blue, 0 }  ][line width=0.08]  [draw opacity=0] (12,-3) -- (0,0) -- (12,3) -- cycle    ;
\draw    (150,50) -- (150,78) ;
\draw [shift={(150,80)}, rotate = 270] [fill={rgb, 255:red, 0; green, 0; blue, 0 }  ][line width=0.08]  [draw opacity=0] (12,-3) -- (0,0) -- (12,3) -- cycle    ;
\draw    (10,96) -- (10,35) -- (48,35) ;
\draw [shift={(50,35)}, rotate = 180] [fill={rgb, 255:red, 0; green, 0; blue, 0 }  ][line width=0.08]  [draw opacity=0] (12,-3) -- (0,0) -- (12,3) -- cycle    ;
\draw  [dash pattern={on 4.5pt off 4.5pt}]  (200,136) -- (200,148.8) -- (200,169) ;

\draw (70,35) node   [align=left] {$\displaystyle \filterpol $};
\draw (70,95) node   [align=left] {$\displaystyle \pol $};
\draw (150,35) node   [align=left] {$\displaystyle \hyperplaneTransformationFunction$};
\draw (150,95) node   [align=left] {$\displaystyle \filter _{\filterpol}^{\pol }$};
\draw (230,95) node   [align=left] {$\displaystyle \dynfcn$};
\draw (229.98,140) node   [align=left] {$\displaystyle \state_{t+1}$};
\draw (150.04,140) node   [align=left] {$\displaystyle \state_{t}$};
\draw (110,81) node   [align=left] {$\displaystyle \act_{t}$};
\draw (110,21) node   [align=left] {$\displaystyle \asf_{t}$};
\draw (126,61) node   [align=left] {$\displaystyle \hyperplaneVector_{t} ,\hyperplaneOffset_{t}$};
\draw (188.5,81) node   [align=left] {$\displaystyle \filteredAction_{t}$};
\draw (237.91,27.25) node  [font=\large] [align=left] {$\displaystyle \textcolor[rgb]{0.09,0.75,0.81}{\dynfcnsf}$};

\end{tikzpicture}
} \fi
    \end{minipage}\hfill
    \begin{minipage}[b]{0.45\textwidth}
        \centering
        \iflatexml\includegraphics[width=0.45\textwidth]{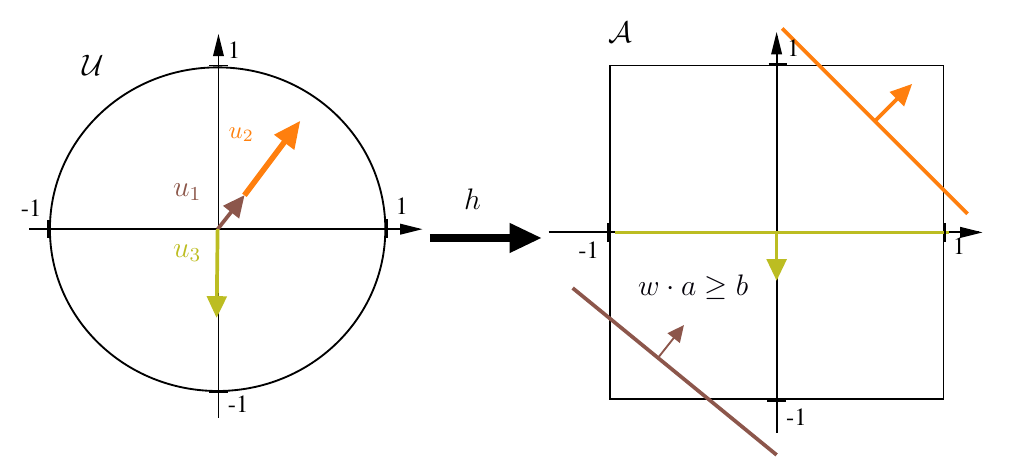}\else 
\tikzset{
pattern size/.store in=\mcSize, 
pattern size = 5pt,
pattern thickness/.store in=\mcThickness, 
pattern thickness = 0.3pt,
pattern radius/.store in=\mcRadius, 
pattern radius = 1pt}
\makeatletter
\pgfutil@ifundefined{pgf@pattern@name@_8g72u0akp}{
\pgfdeclarepatternformonly[\mcThickness,\mcSize]{_8g72u0akp}
{\pgfqpoint{0pt}{0pt}}
{\pgfpoint{\mcSize+\mcThickness}{\mcSize+\mcThickness}}
{\pgfpoint{\mcSize}{\mcSize}}
{
\pgfsetcolor{\tikz@pattern@color}
\pgfsetlinewidth{\mcThickness}
\pgfpathmoveto{\pgfqpoint{0pt}{0pt}}
\pgfpathlineto{\pgfpoint{\mcSize+\mcThickness}{\mcSize+\mcThickness}}
\pgfusepath{stroke}
}}
\makeatother

\tikzset{
pattern size/.store in=\mcSize, 
pattern size = 5pt,
pattern thickness/.store in=\mcThickness, 
pattern thickness = 0.3pt,
pattern radius/.store in=\mcRadius, 
pattern radius = 1pt}
\makeatletter
\pgfutil@ifundefined{pgf@pattern@name@_djn5wu5gt}{
\pgfdeclarepatternformonly[\mcThickness,\mcSize]{_djn5wu5gt}
{\pgfqpoint{0pt}{0pt}}
{\pgfpoint{\mcSize+\mcThickness}{\mcSize+\mcThickness}}
{\pgfpoint{\mcSize}{\mcSize}}
{
\pgfsetcolor{\tikz@pattern@color}
\pgfsetlinewidth{\mcThickness}
\pgfpathmoveto{\pgfqpoint{0pt}{0pt}}
\pgfpathlineto{\pgfpoint{\mcSize+\mcThickness}{\mcSize+\mcThickness}}
\pgfusepath{stroke}
}}
\makeatother

\tikzset{
pattern size/.store in=\mcSize, 
pattern size = 5pt,
pattern thickness/.store in=\mcThickness, 
pattern thickness = 0.3pt,
pattern radius/.store in=\mcRadius, 
pattern radius = 1pt}
\makeatletter
\pgfutil@ifundefined{pgf@pattern@name@_zps63uxue}{
\pgfdeclarepatternformonly[\mcThickness,\mcSize]{_zps63uxue}
{\pgfqpoint{0pt}{0pt}}
{\pgfpoint{\mcSize+\mcThickness}{\mcSize+\mcThickness}}
{\pgfpoint{\mcSize}{\mcSize}}
{
\pgfsetcolor{\tikz@pattern@color}
\pgfsetlinewidth{\mcThickness}
\pgfpathmoveto{\pgfqpoint{0pt}{0pt}}
\pgfpathlineto{\pgfpoint{\mcSize+\mcThickness}{\mcSize+\mcThickness}}
\pgfusepath{stroke}
}}
\makeatother
\tikzset{every picture/.style={line width=0.75pt}} %
\resizebox{\linewidth}{!}{
\begin{tikzpicture}[x=0.75pt,y=0.75pt,yscale=-1,xscale=1]

\draw  [draw opacity=0][pattern=_8g72u0akp,pattern size=6pt,pattern thickness=0.75pt,pattern radius=0pt, pattern color={rgb, 255:red, 44; green, 160; blue, 44}] (424,10) -- (522.36,108.22) -- (534.2,96.36) -- (435.85,-1.86) -- cycle ;
\draw  [draw opacity=0][pattern=_djn5wu5gt,pattern size=6pt,pattern thickness=0.75pt,pattern radius=0pt, pattern color={rgb, 255:red, 44; green, 160; blue, 44}] (487,118) -- (348,118) -- (348,134.76) -- (487,134.76) -- cycle ;
\draw  [draw opacity=0][pattern=_zps63uxue,pattern size=6pt,pattern thickness=0.75pt,pattern radius=0pt, pattern color={rgb, 255:red, 44; green, 160; blue, 44}] (309.57,148.79) -- (417,237) -- (427.64,224.04) -- (320.21,135.83) -- cycle ;
\draw   (25.11,115.26) .. controls (25.11,67.07) and (65.6,28) .. (115.55,28) .. controls (165.5,28) and (205.99,67.07) .. (205.99,115.26) .. controls (205.99,163.46) and (165.5,202.53) .. (115.55,202.53) .. controls (65.6,202.53) and (25.11,163.46) .. (25.11,115.26) -- cycle ;
\draw    (14,115.26) -- (224,115.26) ;
\draw [shift={(226,115.26)}, rotate = 180] [fill={rgb, 255:red, 0; green, 0; blue, 0 }  ][line width=0.08]  [draw opacity=0] (12,-3) -- (0,0) -- (12,3) -- cycle    ;
\draw    (116,217) -- (116,12) ;
\draw [shift={(116,10)}, rotate = 90] [fill={rgb, 255:red, 0; green, 0; blue, 0 }  ][line width=0.08]  [draw opacity=0] (12,-3) -- (0,0) -- (12,3) -- cycle    ;
\draw [color={rgb, 255:red, 140; green, 86; blue, 75 }  ,draw opacity=1 ][line width=1.5]    (115.55,115.26) -- (127.52,100.14) ;
\draw [shift={(130,97)}, rotate = 128.36] [fill={rgb, 255:red, 140; green, 86; blue, 75 }  ,fill opacity=1 ][line width=0.08]  [draw opacity=0] (11.61,-5.58) -- (0,0) -- (11.61,5.58) -- cycle    ;
\draw [color={rgb, 255:red, 255; green, 127; blue, 14 }  ,draw opacity=1 ][line width=2.25]    (130,97) -- (157,61) ;
\draw [shift={(160,57)}, rotate = 126.87] [fill={rgb, 255:red, 255; green, 127; blue, 14 }  ,fill opacity=1 ][line width=0.08]  [draw opacity=0] (14.29,-6.86) -- (0,0) -- (14.29,6.86) -- cycle    ;
\draw [color={rgb, 255:red, 188; green, 189; blue, 34 }  ,draw opacity=1 ][line width=1.5]    (115.55,115.26) -- (115.05,159) ;
\draw [shift={(115,163)}, rotate = 270.66] [fill={rgb, 255:red, 188; green, 189; blue, 34 }  ,fill opacity=1 ][line width=0.08]  [draw opacity=0] (11.61,-5.58) -- (0,0) -- (11.61,5.58) -- cycle    ;
\draw    (111,27) -- (121,27) ;
\draw    (111,203) -- (121,203) ;
\draw [line width=3]    (230,120) -- (284,120) ;
\draw [shift={(290,120)}, rotate = 180] [fill={rgb, 255:red, 0; green, 0; blue, 0 }  ][line width=0.08]  [draw opacity=0] (16.97,-8.15) -- (0,0) -- (16.97,8.15) -- cycle    ;
\draw   (327,27) -- (507,27) -- (507,207) -- (327,207) -- cycle ;
\draw    (417,225) -- (417,11) ;
\draw [shift={(417,9)}, rotate = 90] [fill={rgb, 255:red, 0; green, 0; blue, 0 }  ][line width=0.08]  [draw opacity=0] (12,-3) -- (0,0) -- (12,3) -- cycle    ;
\draw    (294,117) -- (526,117) ;
\draw [shift={(528,117)}, rotate = 180] [fill={rgb, 255:red, 0; green, 0; blue, 0 }  ][line width=0.08]  [draw opacity=0] (12,-3) -- (0,0) -- (12,3) -- cycle    ;
\draw [color={rgb, 255:red, 140; green, 86; blue, 75 }  ,draw opacity=1 ]   (352.55,185.26) -- (365.14,169.35) ;
\draw [shift={(367,167)}, rotate = 128.36] [fill={rgb, 255:red, 140; green, 86; blue, 75 }  ,fill opacity=1 ][line width=0.08]  [draw opacity=0] (8.93,-4.29) -- (0,0) -- (8.93,4.29) -- cycle    ;
\draw [color={rgb, 255:red, 140; green, 86; blue, 75 }  ,draw opacity=1 ][line width=1.5]    (417,237) -- (307,147) ;
\draw [color={rgb, 255:red, 255; green, 127; blue, 14 }  ,draw opacity=1 ][line width=1.5]    (470,57) -- (487.17,39.83) ;
\draw [shift={(490,37)}, rotate = 135] [fill={rgb, 255:red, 255; green, 127; blue, 14 }  ,fill opacity=1 ][line width=0.08]  [draw opacity=0] (11.61,-5.58) -- (0,0) -- (11.61,5.58) -- cycle    ;
\draw [color={rgb, 255:red, 255; green, 127; blue, 14 }  ,draw opacity=1 ][line width=1.5]    (520,107) -- (420,7) ;
\draw [color={rgb, 255:red, 188; green, 189; blue, 34 }  ,draw opacity=1 ][line width=1.5]    (417,117) -- (417,139) ;
\draw [shift={(417,143)}, rotate = 270] [fill={rgb, 255:red, 188; green, 189; blue, 34 }  ,fill opacity=1 ][line width=0.08]  [draw opacity=0] (11.61,-5.58) -- (0,0) -- (11.61,5.58) -- cycle    ;
\draw [color={rgb, 255:red, 188; green, 189; blue, 34 }  ,draw opacity=1 ][line width=1.5]    (330,117) -- (510,117) ;
\draw    (24.11,110.26) -- (24.11,120.26) ;
\draw    (207,110) -- (207,120) ;
\draw    (412.63,26) -- (422.63,26) ;
\draw    (412,208) -- (422,208) ;
\draw    (326,112) -- (326,122) ;
\draw    (508,112) -- (508,122) ;

\draw (90,89) node [anchor=north west][inner sep=0.75pt]  [font=\large] [align=left] {$\displaystyle \textcolor[rgb]{0.55,0.34,0.29}{\asf}\textcolor[rgb]{0.55,0.34,0.29}{_{1}}$};
\draw (120,59) node [anchor=north west][inner sep=0.75pt]  [color={rgb, 255:red, 255; green, 127; blue, 14 }  ,opacity=1 ] [align=left] {$\displaystyle \textcolor[rgb]{1,0.5,0.05}{\asf}\textcolor[rgb]{1,0.5,0.05}{_{\textcolor[rgb]{1,0.5,0.05}{2}}}$};
\draw (90,122) node [anchor=north west][inner sep=0.75pt]  [font=\large] [align=left] {$\displaystyle \textcolor[rgb]{0.74,0.74,0.13}{\asf}\textcolor[rgb]{0.74,0.74,0.13}{_{3}}$};
\draw (210,97) node [anchor=north west][inner sep=0.75pt]   [align=left] {1};
\draw (124.38,18.5) node   [align=left] {1};
\draw (126.67,209.5) node   [align=left] {\mbox{-}1};
\draw (8,98) node [anchor=north west][inner sep=0.75pt]   [align=left] {\mbox{-}1};
\draw (247,92) node [anchor=north west][inner sep=0.75pt]  [font=\large] [align=left] {$\displaystyle \hyperplaneTransformationFunction$};
\draw (427.67,216.5) node   [align=left] {\mbox{-}1};
\draw (315.67,126.5) node   [align=left] {\mbox{-}1};
\draw (515.38,124.5) node   [align=left] {1};
\draw (325,2) node [anchor=north west][inner sep=0.75pt]  [font=\large] [align=left] {$\displaystyle \actsp$};
\draw (372,146.5) node  [font=\large,color={rgb, 255:red, 3; green, 0; blue, 8 }  ,opacity=1 ] [align=left] {$\displaystyle \textcolor[rgb]{0.01,0,0.03}{\hyperplaneVector\cdot \act \geq \hyperplaneOffset}$};
\draw (41,20) node [anchor=north west][inner sep=0.75pt]  [font=\large] [align=left] {$\displaystyle \actspsf$};
\draw (426,17.5) node   [align=left] {1};

\end{tikzpicture}

} \fi
    \end{minipage}

    \begin{minipage}[t]{0.45\textwidth}
          \caption{Safety Filter MDP Dynamics $\dynfcnsf$. \emph{At state $\state_t$, the control policy $\pol$ proposes action $\act_t$. The safety filter $\filter$ \eqref{eq:learned_filter}, parametrized by the filter policy $\filterpol$ via the bijective transform $\hyperplaneTransformationFunction$ and the hyperplane action $\asf_t$, outputs a viable action $\viaact_t$ applied to dynamics $\dynfcn$.
          }}
          \label{fig:filter_mdp} 
    \end{minipage}\hfill
    \begin{minipage}[t]{0.45\textwidth}
         \caption{Hyperplane Action Space $\actspsf$. \\\emph{Example mappings between hyperplane actions $\asf \in \actspsf$ and their corresponding discriminating hyperplanes in the control action space $\actsp$. Larger norms $\Vert\asf\Vert_2$  of the hyperplane action result in more restrictive hyperplanes.}}
        \label{fig:hyperplane_parameterization}
    \end{minipage}
\end{figure}

To cast the filter objective \eqref{eq:filter_learning_problem} as a reinforcement learning problem, we define the safety filter MDP $\mdpsf = (\statesp,\actspsf,\dynfcnsf, \rewfcnsf, \initdistsf, \discountsf)$, which shares only the state space $\statesp$ with the control MDP $\mdp$ from Section \ref{subsec:mbrl}. 
The filter policy $\filterpol: \statesp \rightarrow \actspsf$ outputs hyperplane actions $\asf_t \in \actspsf$, which a bijection $\hyperplaneTransformationFunction$ maps to $\hyperplaneVector_t$ and $\hyperplaneOffset_t$. We discuss the benefits of this design in Section \ref{subsec:safe_act_formulation}. As depicted in Figure \ref{fig:filter_mdp}, the dynamics $\dynfcnsf: \statesp \times \actspsf \rightarrow \Delta(\statesp)$ are composed of any control policy $\pol \in \polset$, the environment dynamics $\dynfcn$ approximated by $\moddyn$, the transform $\hyperplaneTransformationFunction$, and the safety filter $\filter$. 
The reward function $\rewfcnsf$ encourages the least restrictive projection $\polset \rightarrow \safepolset(\accurset{})$ as discussed in Section \ref{subsec:safe_rew_formulation}.
As $\filterpol$ is trained on model-based rollouts, its initial state distribution $\initdistsf$ does not need to match $\initdist$ and is instead tailored to be more informative as presented in Section \ref{subsec:safe_dist_formulation}. Finally, the discount factor $\discountsf$ defines the effective horizon $T$ in \eqref{eq:viable_pol_in_model} by inducing a random stopping time with expectation $\mathbb{E}[T] = \frac{1}{1-\discountsf}$.

\subsection{Efficient Filter Action Space Formulation}
\label{subsec:safe_act_formulation}
Addressing question \emph{(ii)}, we present a novel, expressive design of $\actspsf$. \citet{black-box-lavankul} propose learning hyperplane parameters directly as $\tilde{\asf}_t = [\hyperplaneVector_t^\top, \hyperplaneOffset_t]^\top$, which causes several problems. First, this formulation is unbounded with $\tilde{\actspsf} = \mathbb{R}^{\na + 1}$ and overparametrized since $c\,\hyperplaneVector_t^\top \cdot \act \geq c\,\hyperplaneOffset_t$ yields identical hyperplanes for any $c \in \mathbb{R}_{>0}$. Second, $\tilde{\actspsf}$ does not enforce hyperplane intersection with the bounded control action space $\actsp$. These issues create a large, ambiguous search space.

Instead, we define the space of hyperplane actions as the unit hyperball in $\mathbb{R}^{\na}$
\begin{equation}
\label{eq:hyperplaneactsp}
\actspsf := \left\{\asf \in \mathbb{R}^{\na} \,\middle|\, 0<\|\asf\|_2 < 1 \right\},
\end{equation}
resulting in a bounded search space. Subsequently, we obtain the hyperplane parameters via the bijective function $\hyperplaneTransformationFunction: \actspsf \rightarrow \R^{\na} \times \R$ that enforces intersection with $\actsp$.

\begin{theorem}
The function $\hyperplaneTransformationFunction$ mapping each $\asf \in \actspsf$ to a discriminating hyperplane intersecting the interior of $\actsp = [-1, 1]^{\na}$ is bijective. It is defined as $\hyperplaneTransformationFunction(\asf_t) = (\hyperplaneVector_t,\hyperplaneOffset_t)$ with
\begin{equation}
    \hyperplaneVector_t = \frac{\asf_t}{\Vert \asf_t \Vert_2}, \text{ and } \hyperplaneOffset_t = (2\Vert \asf_t \Vert_2 -1 ) \Vert  \hyperplaneVector_t \Vert _1.
\end{equation}
\proof{See Appendix~\ref{app:proofTransformation}, Theorem \ref{theo:bijection_app}}
\end{theorem}

Figure \ref{fig:hyperplane_parameterization} illustrates the transformation\footnote{Interactive demo of the mapping between hyperplane action and the coresponding hyperplane for a two-dimensional control action space: \url{https://data-science-in-mechanical-engineering.github.io/Dyna-SAuR/\#demo}.}. The normal vector $\hyperplaneVector_t$ points in the direction of $\asf_t$, where actions $\act_t$ in the positive direction of $\hyperplaneVector_t$ are classified as viable. The offset $\hyperplaneOffset_t$ is computed using the distance from the origin to the farthest vertex of $\actsp$ along $\hyperplaneVector_t$, given by $\Vert \hyperplaneVector_t \Vert_1$ as shown in Lemma \ref{lemma:MaxOffset}. We scale this distance by $(2\Vert \asf_t \Vert_2 - 1) \in (-1, 1)$ by construction in \eqref{eq:hyperplaneactsp}.
Thus, $\Vert \asf_t \Vert_2$ indicates restrictiveness: 
values near $0$ classify most actions viable, while values close to $1$ reject most.
This yields a minimal search space for $\filterpol$, improving the scalability of data-driven safety filter learning.

\subsection{Filter Policy Learning as an RL Problem}
\label{subsec:safe_rew_formulation}
As we aim for a general reward function that enforces staying within $\safesetSFi{}$ under $\dynfcnsf$, we define
\begin{equation}\label{eq:reward_func_safe}
\rewfcnsf(\state_t,\asf_t) =
\begin{cases}
1, & \text{if } (\state_{t}, \viaact_t) \in \safesetSFi{}, \\
-\frac{1}{1-\discountsf}, & \text{otherwise}
\end{cases}
\end{equation}
where $\act^\mathrm{V}_t$ results from $\asf_t$ via the learned filter equation \eqref{eq:learned_filter}. Further, we terminate whenever $\safesetSFi{}$ is left. As a consequence, the safety filter action value function
\begin{equation}
    \filterqfcn(\state_t, \asf_t) := \mathbb{E}_{\asf \sim \filterpol, \state \sim \dynfcnsf}\left[ \sum_{k=0}^\infty \discountsf^{t+k} \rewsf_{t+k+1} \Big| \state_t, \asf_t \right]
\end{equation}
is bounded with $\filterqfcn(\state_t, \asf_t) \in [-\frac{1}{1-\discountsf}, \frac{1}{1-\discountsf}]$ for all state-filter action pairs $(\state_t, \asf_t)$. Here, a value of $-\frac{1}{1-\discountsf}$ corresponds to failing in the next transition, and $\frac{1}{1-\discountsf}$ corresponds to not failing at all. In particular, safety filter action values provide a lower bound on the expected time to failure
\begin{equation}
   \mathbb{E}_\filterpol \lb T_{\mathrm{F}}| \state \rb \geq \log_{\discountsf}\left( \frac{1-\filterqfcn(\state, \filterpol(\state))(1-\discountsf)}{2}\right)
\end{equation}
under $\filterpol$ as derived in Lemma \ref{lemma:time_fail} of Appendix \ref{app:time_to_failure}.
Thus, maximizing $\filterqfcn$ enforces the system to stay within $\safesetSFi{}$ for as long as possible, which closely reflects \eqref{eq:viable_pol_in_model}.

To obtain a least restrictive projection $\polset \rightarrow \safepolset(\accurset{})$, we leverage the fact that $\| \asf_t \|_2$ indicates filter restrictiveness, with lower norms corresponding to more viable actions. Thus, maximizing
\begin{equation}
    \filterpol(\state_t) \leftarrow \arg\max_{\filterpol \in \filterpolset} \filterqfcn\left(\state_t, \filterpol(\state_t)\right) - c \| \filterpol(\state_t) \|_2
    \label{eq:rl_filter_learning_objective}
\end{equation}
with a regularization $\| \filterpol(\state_t) \|_2$ scaled by $c$, closely resembles the filter policy learning problem \eqref{eq:filter_learning_problem}.

\subsection{Addressing Stochasticity via Data Shaping}
\label{subsec:safe_dist_formulation}
Solving the RL objective \eqref{eq:rl_filter_learning_objective} in practice is challenging due to the randomness in the dynamics $\dynfcnsf$ depicted in Figure \ref{fig:filter_mdp}. While the filter components $\hyperplaneTransformationFunction$ and $\kappa$ are deterministic, and the control MDP dynamics $\dynfcn$ typically introduce mild stochasticity, allowing any $\pol \in \polset$ makes $\dynfcnsf$ highly stochastic. Thus, areas in $\statesp \times \actspsf$ need to be heavily sampled to obtain reliable value estimates $\filterqfcn(\state_t, \asf_t)$ and to train an effective filter policy. Learning an expressive $\filterqfcn$ everywhere, akin to Hamilton-Jacobi reachability methods, is infeasible for systems of the considered complexity ~\cite{wabersich2023data}. Instead, we leverage two principles that helped scale dynamic programming to modern deep RL: large-scale, parallelized data generation, and concentrating data and representational capacity on relevant regions of $\statesp \times \actspsf$. We leverage accurate long-horizon rollouts using the model $\moddyn$ \cite{Frauenknecht2025Infoprop} to cheaply generate large amounts of data,  and actor-critic algorithms tailored for large datasets \cite{fasttd3}. Addressing question $\emph{(iii)}$, we present heuristics for concentrating sampling budget in meaningful areas of $\statesp \times \actspsf$.

Generally, RL methods with function approximation minimize the mean squared value error \cite{Sutton1998} 
\begin{equation}
    \xi := \int_{\statesp} \int_{\actspsf} \rho(\state, \asf) \left( q^{\mathrm{SF}}(\state, \asf) - \filterqfcn(\state, \asf)  \right)^2 \, \mathrm{d} \asf \,\mathrm{d}\state
    \label{eq:value_error}
\end{equation}
where $q^{\mathrm{SF}}$ is the true filter action value function and $\rho(\state, \asf)$ is the state-action distribution used for training. Equation~\eqref{eq:value_error} shows that $\filterqfcn$ is generally more accurate in high-density regions of $\rho(\state, \asf)$. This distribution depends on $\initdistsf$, $\filterpol$, and $\dynfcnsf$. We shape $\rho(\state, \asf)$ by oversampling informative states in the initial state distribution $\initdistsf$ and finding a suitable representation for the stochasticity in $\dynfcnsf$.

We use three sources for starting states $\initdistsf$ in model-based rollouts. First, $\initdistsf$ must include the environment distribution $\rho_0$, approximated from $\buffer_0^{\moddyn, \dynfcn}$, to mimic rollouts in the environment.
Second, failed rollouts $\trajfail$ that left $\safesetSFi{}$ should be overrepresented, since $\filterpol$ does not yet provide a valid safety filter in these areas and requires more training. Third, rollouts $\trajret$ yielding high returns concerning the control reward $\rewfcn$ should be overrepresented, since an overly conservative filter would be especially harmful in these areas. Thus, we want to make sure that the filter policy RL objective \eqref{eq:rl_filter_learning_objective} converges concerning the regularization loss $\| \asf_t\|_2$ and finds a least restrictive projection. This yields
\begin{equation}
    \initdistsf = \nu_1 \initdist + \nu_2 \unidist(\trajfailset) + \nu_3 \unidist(\trajretset)
\end{equation}
with $\sum_{i=1}^3 \nu_i = 1$ and $\nu_i \geq 0$, which is the mixture distribution over $\initdist$ and uniformly sampling states from the trajectory sets $\trajfail \in \trajfailset$ and $\trajret \in \trajretset$. 

Next, we face the challenge of approximating the stochasticity of $\dynfcnsf$ induced by $\pol \in \polset$. Since $\polset$ typically is a continuous and unbounded search space, sampling policies at random is unlikely to yield satisfactory results given a finite sampling budget. Thus, instead of considering the policy space directly, we focus on the state occupancy corresponding to $\polset$. We define this occupancy $\rho_{\pol \in \polset}(\state, \act) := \mathbb{E}_{\pol \in \polset} [\rho_{\pol}(\state, \act)]$ as the expected occupancy over all $\pol \in \polset$. Further, the state occupancy of certain viable policies $\rho_{\pol \in \safepolset(\accurset{})}(\state, \act) := \mathbb{E}_{\pol \in \safepolset(\accurset{})} [\rho_{\pol}(\state, \act)]$ has no probability mass outside of $\safesetSFi{}$.

A filter $\filter_\filterpol^\pol$ projecting any policy $\pol \in \polset$ to a certain viable policy in $\safepolset(\accurset{})$, also represents a mapping between the corresponding state occupancies, such that 
\begin{equation}
    \filter_\filterpol^\pol : \polset \rightarrow \safepolset(\accurset{}) \iff \filter_\filterpol^\pol : \rho_{\pol \in \polset} \rightarrow \rho_{\pol \in \safepolset(\accurset{})}.
\end{equation}
Thus, for learning the filter policy, it suffices to find a projection into $\rho_{\pol \in \safepolset(\accurset{})}$ that maps all probability mass into $\safesetSFi{}$, regardless of whether the specific control actions originate from a policy within $\polset$.

This insight allows us to approximate $\rho_{\pol \in \polset}$, the filter policy needs to project into $\rho_{\pol \in \safepolset(\accurset{})}$ using
\begin{equation}
    \rho_{\pol \in \polset}(\state, \act) \approx \nu_4 \rho_{\pol_{\algiter-1} + \expscale_2 \noise}(\state, \act) + \nu_5 \rho_{\expscale_3 \noise}(\state, \act),
\end{equation}
with $\nu_4 + \nu_5 = 1$ and $\nu_4, \nu_5 \geq 0$. This represents a weighted sum of the state occupancies corresponding to the current control policy $\pol_{\algiter-1}$ with exploration noise $\noise$ and the pure noise $\noise$, with scaling factors $\expscale_2, \expscale_3 \in \mathbb{R}$. We use temporally correlated pink noise $\noise \sim \mathrm{PN}$ to balance local and global exploration \cite{eberhard-2023-pink}. This design yields a diverse approximation of $\rho_{\pol \in \polset}$ while emphasizing the current best policy $\pol_{\algiter-1}$, ensuring the resulting filter performs well for greedy control actions.

\section{Experiments and Discussion}
\label{sec:experiments}
\iflatexml
\begin{wrapfigure}{r}{0.45\textwidth}
    \centering
    \includegraphics[width=0.45\textwidth]{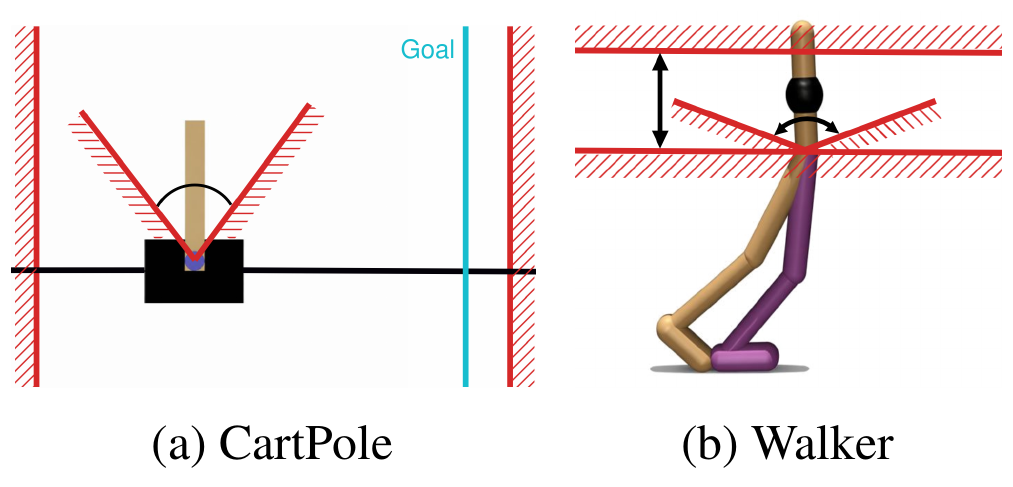}
    \caption{Safe Learning Environments with Constraints (red) and Goal (blue).}
    \label{fig:benchmarked_envs}
\end{wrapfigure}
\else
\begin{wrapfigure}[11]{r}{0.45\textwidth}
    \centering
    \begin{subfigure}[t]{0.53\linewidth}   \centering\includegraphics[width=\linewidth]{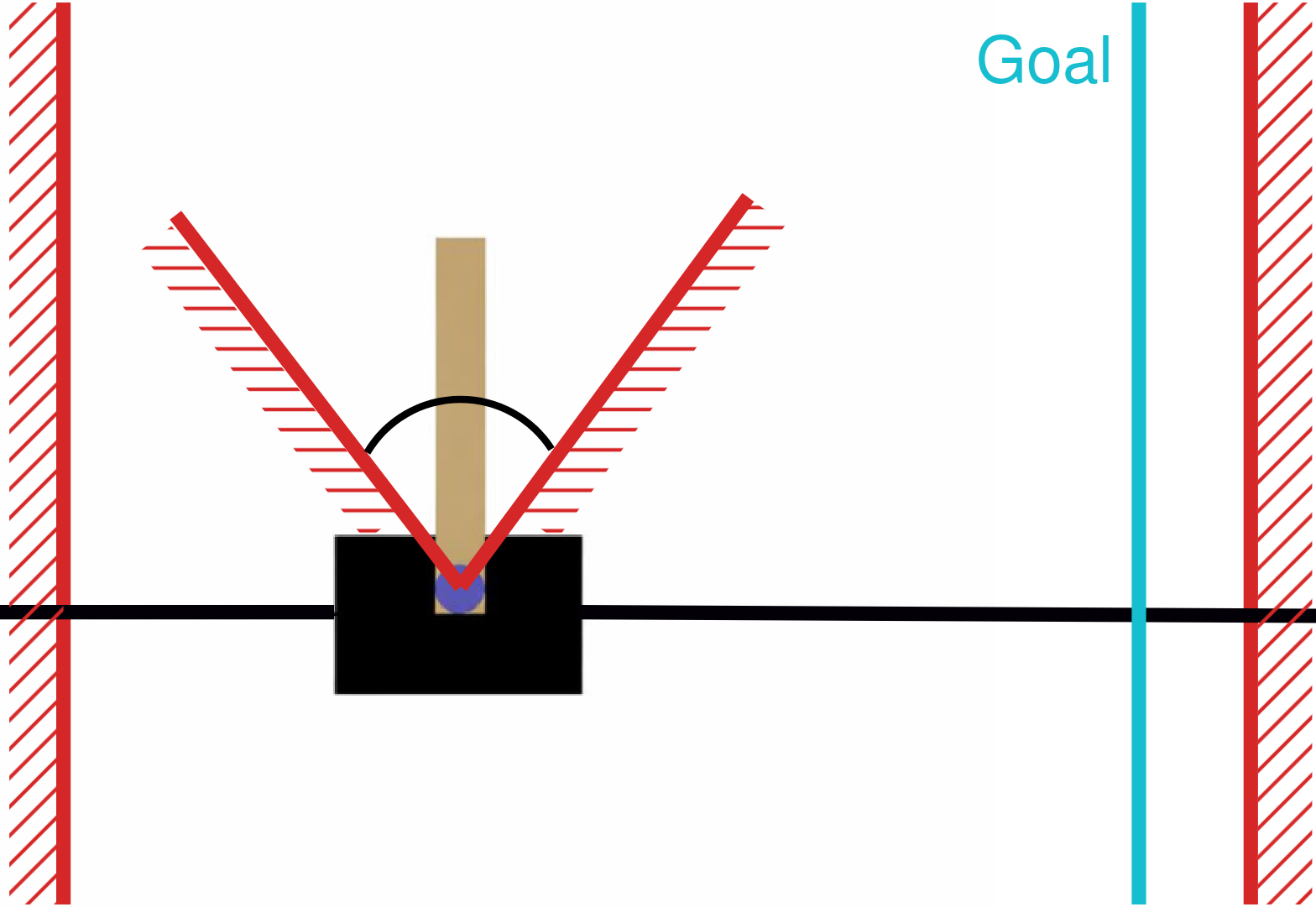}
        \caption{CartPole}
        \label{fig:cartpole_safety}
    \end{subfigure}
    \hfill
    \begin{subfigure}[t]{0.43\linewidth}
        \centering
        \includegraphics[width=\linewidth]{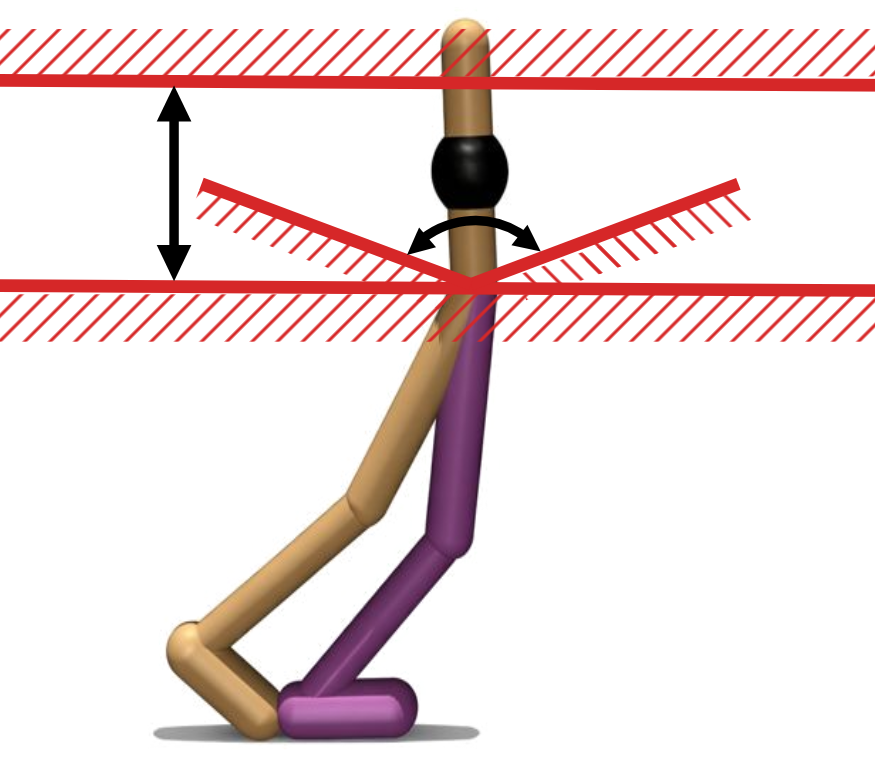}
        \caption{Walker}
        \label{fig:walker_safety}
    \end{subfigure}
    \caption{Safe Learning Environments with Constraints (red) and Goal (blue).}
    \label{fig:benchmarked_envs}
\end{wrapfigure}
\fi
We evaluate \dynasaur{} on common RL benchmarks that we extend with constraints that must not be violated, as depicted in Figure \ref{fig:benchmarked_envs}.\footnote{
All code is provided under \url{https://github.com/Data-Science-in-Mechanical-Engineering/Dyna-SAuR}.} We consider a constraint-violating state unsafe, while \dynasaur{} needs to learn the distinction between viable and unviable actions without direct access to the environment dynamics. Specifically, we present results on a  goal-reaching CartPole task ($\statesp \subseteq \R^4, \actsp \subseteq \R$) and MuJoCo Walker ($\statesp \subseteq \R^{17}, \actsp \subseteq \R^6$). 

Goal reaching CartPole extends the standard task \cite{towers2025gymnasium} with a goal, such that the CartPole needs to leave its equilibrium to reach the goal state at the physical limits without constraint violation. Further, Walker~\cite{todorov_mujoco_2012} is a relatively high dimensional problem with discontinuous contact dynamics that is beyond the scaling capabilities of most viability-based safety filter approaches. Thus, we compare \dynasaur{} to scalable established safe RL baselines, where it yields comparable or superior control performance while reducing accumulated failures during training by at least two orders of magnitude.

\subsection{Experimental Setup}\label{subsec:expSetup}
\begin{figure}[t]
          \centering
\includegraphics[width=1\linewidth]{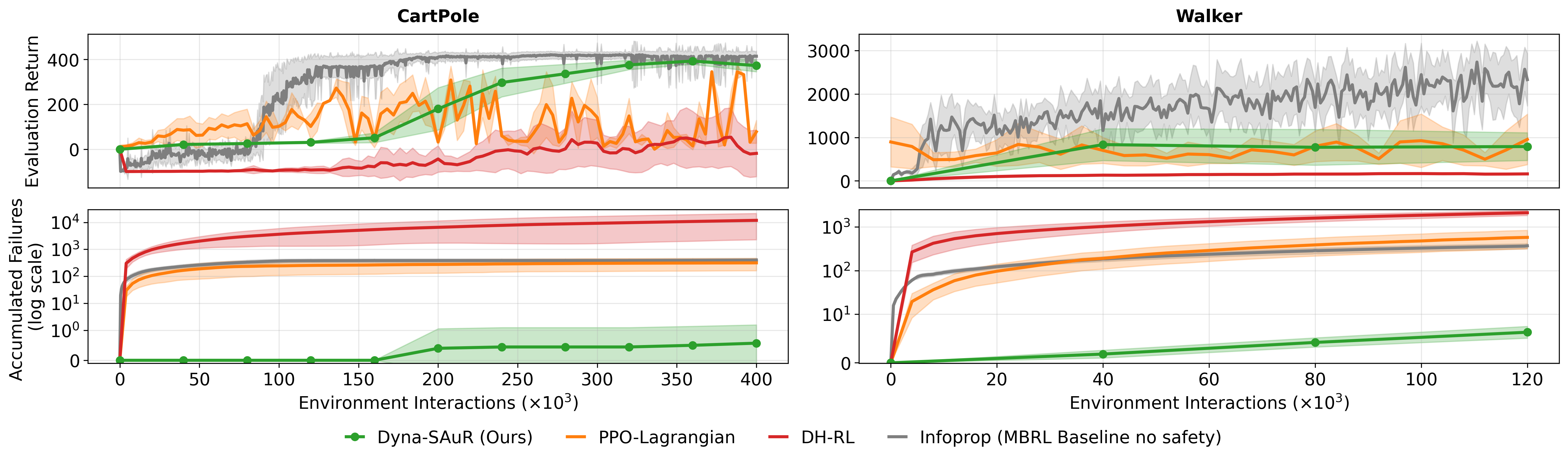}
    \caption{Control Return and Accumulated Failures (log scale).  \emph{Performance is measured after each \dynasaur{} iteration, as the agent is retrained from scratch. The results demonstrate that \dynasaur{} matches or exceeds the performance of safe RL algorithms (PPO-Lagrangian, DH-RL) while significantly reducing the number of failures compared to all approaches.}}
    \label{fig:resultsReturnFailureDynasaur}
\end{figure}

We compare \dynasaur{} to \emph{(a)} PPO-Lagrangian~\cite{Ray2019SafeExploration}, an established model-free safe RL approach that minimizes failures in expectation; \emph{(b)} DH-RL~\cite{black-box-lavankul}, a safe RL method that learns a safety filter policy through interaction with a simulator; and \emph{(c)} Infoprop-Dyna~\cite{Frauenknecht2025Infoprop}, a recent MBRL method without safety mechanisms. \dynasaur{} is an extension to DH-RL, as it learns a filter policy with action space $\actspsf$ (Section \ref{subsec:safe_act_formulation}) from a model, and to Infoprop-Dyna as it extends MBRL to safe RL.

We provide the initial dataset $\buffer^{\moddyn, \dynfcn}_0$ to both \dynasaur{} and Infoprop-Dyna.
Since PPO-Lagrangian and DH-RL cannot directly operate on offline data, we warmstart them with the same number of online environment interactions as stored in $\buffer^{\moddyn, \dynfcn}_0$ before reporting performance and counting safety violations.
Appendix~\ref{app:experimental_setup} provides a detailed discussion of the experimental setup, including the generation of $\buffer^{\moddyn, \dynfcn}_0$. Initial data collection is decoupled from \dynasaur{}: the data can be replaced by logged data from any behavior policy, and its generating controllers are not used during safety filter training.
All experiments are averaged over 10 random seeds with solid lines representing the mean and shaded regions the 99\% confidence interval.

\subsection{Control Performance and Safety}\label{subsec:contPerfSafety}
\begin{wrapfigure}[20]{r}{0.45\textwidth}
    \centering
    \vspace{-15pt}
    \includegraphics[width=1\linewidth]{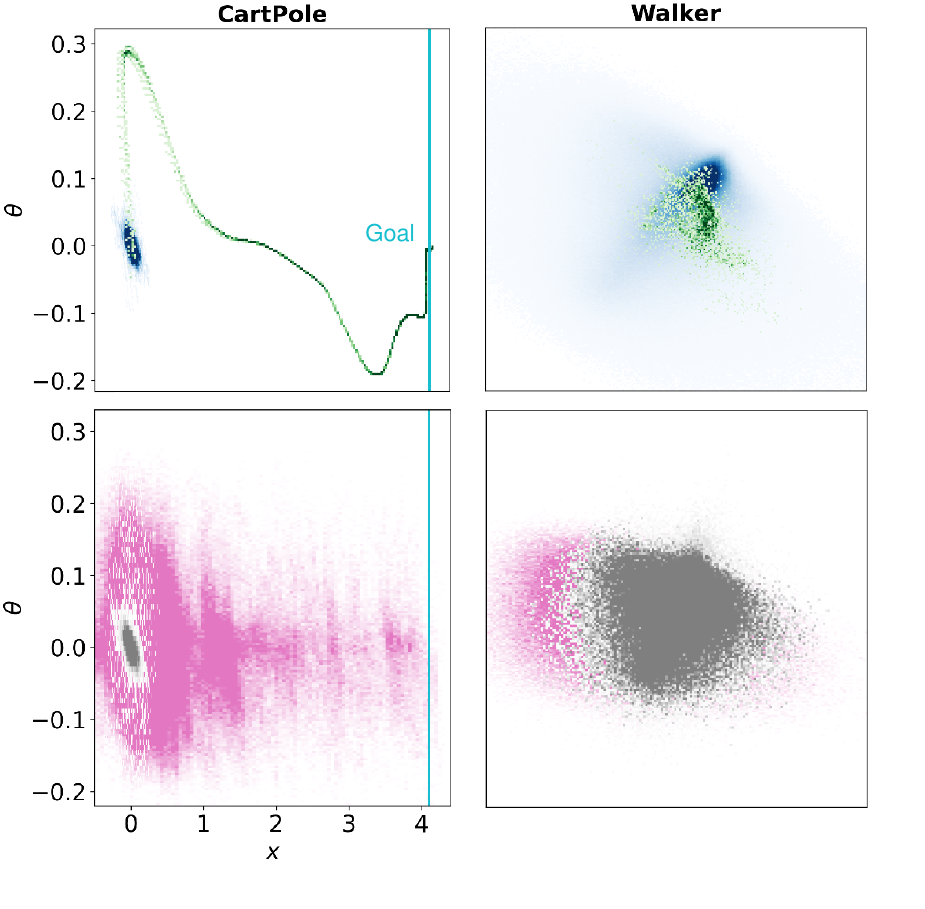}
    \vspace{-15pt}
    \caption{Exploration Throughout Training. \emph{Top row: Initial data $\buffer_0^{\moddyn, \dynfcn}$ (blue) vs. data of the final filtered policy $\filter_{\filterpol_\maxalgiter}^{\pol_\maxalgiter}$ (green); Bottom row: Initial certain set $\accurset{0}$ (grey) vs. final certain set $\accurset{\maxalgiter}$ (pink). \dynasaur{} explores the environment beyond initial data.}}
    \label{fig:dataDistEnvs}
    
\end{wrapfigure}
Figure~\ref{fig:resultsReturnFailureDynasaur} reports control returns (top row) and accumulated failures (bottom row) over environment interactions for both tasks. Since each \dynasaur{} iteration comprises retraining the policy from scratch, only the final evaluation return after each iteration is shown as a green dot to maintain readability. All failures incurred during training are included.
\dynasaur{} shows the strongest control performance of all safe RL approaches on goal-reaching CartPole and performs on par with PPO-Lagrangian on Walker. Infoprop-Dyna consistently outperforms the safe RL approaches concerning control performance, but does not account for safety. \dynasaur{} reduces the amount of accumulated failures by at least two orders of magnitude compared to the other methods.

Figure \ref{fig:dataDistEnvs} investigates the exploration behavior of \dynasaur{}. The top row compares the initial data set $\buffer^{\moddyn, \dynfcn}_0$ to trajectories under the final filtered policy $\filter_{\filterpol_\maxalgiter}^{\pol_\maxalgiter}$, in blue and green, respectively. Similarly, the bottom row compares the initial certain set $\accurset{0}$ in grey to the final certain set $\accurset{\maxalgiter}$  in pink.

For goal-reaching CartPole, we plot the subspace $x \times \theta$, indicating a large disconnect between initial data and the final policy. In particular, the initial data and accurate set are far from the goal state. Consequently, $\accurset{}$ needed to expand substantially, to learn higher performing policies, which explains the gradual increase in control performance in Figure \ref{fig:resultsReturnFailureDynasaur}. Since \dynasaur{} accumulates at most one failure per seed on goal-reaching CartPole, the gradual expansion of $\accurset{}$ yielding less restrictive $\filterpol$ appears to work as intended.

\begin{wrapfigure}[16]{r}{0.45\textwidth}
    \centering
    \vspace{-3pt}
    \includegraphics[width=1\linewidth]{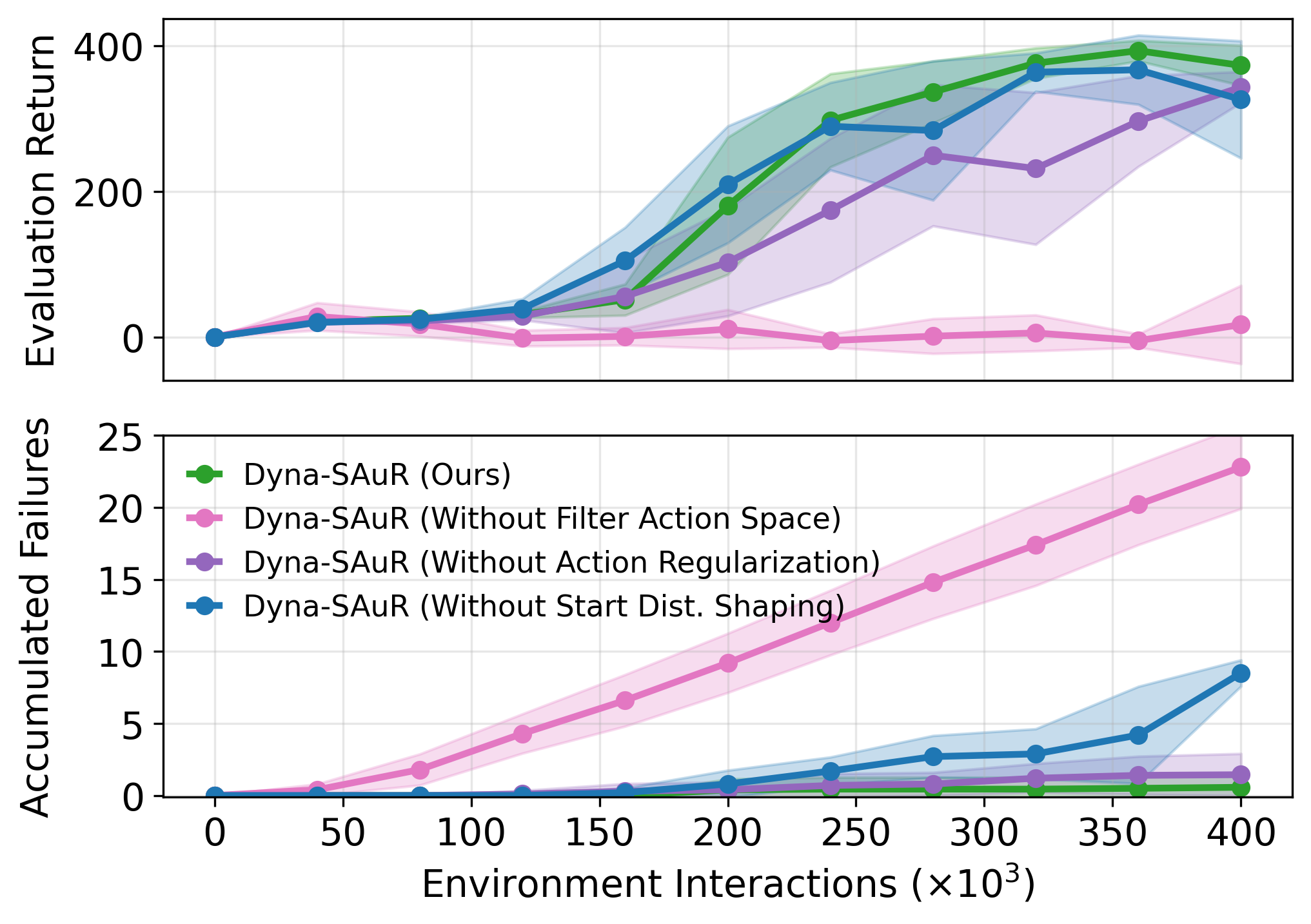}
    \caption{Ablation of \dynasaur{} Design Choices. 
    \emph{
    Only the combination of the compact filter action space (Section \ref{subsec:safe_act_formulation}), the regularization loss (Section~\ref{subsec:safe_rew_formulation}), and the distribution shaping (Section~\ref{subsec:safe_dist_formulation}), yields strong results concerning performance and safety.
    }
    }
    \label{fig:ablation_cartpole}
\end{wrapfigure}
Given the 17-dimensional state space, we perform a principal component analysis to compare the distributions in Walker.
Both the differences between initial data and final policy as well as between the initial and final certain set are substantial, while less pronounced than in CartPole.
Combined with the strong but stagnant control performance in Figure \ref{fig:resultsReturnFailureDynasaur}, this suggests that efficiently expanding $\accurset{}$ in high-dimensional spaces is a limitation of the current approach.

\subsection{Ablation of Filter Policy Design Decisions}\label{subsec:ablation}

We present ablation results of the design decisions presented in Section \ref{sec:method} for goal-reaching CartPole in Figure \ref{fig:ablation_cartpole} with results for Walker in Appendix~\ref{app:ablationWalker}.

First, the predominant importance of formulating the action filter space as presented in Section \ref{subsec:safe_act_formulation} is illustrated when compared to learning $\hyperplaneVector, \hyperplaneOffset$ directly as proposed in \cite{black-box-lavankul}, shown in pink. The method is incapable of learning a reliable safety filter that allows for high control returns. Second, removing the regularization loss on $\asf$ from the objective \eqref{eq:rl_filter_learning_objective} introduced in Section \ref{subsec:safe_rew_formulation}, shown in purple, still yields a good filter but shows weaker control performance. This indicates a valid but restrictive projection into $\safepolset(\accurset{})$, as expected. Third, removing the heuristics for shaping $\initdistsf$ presented in Section \ref{subsec:safe_dist_formulation} and starting from $\initdist$, instead, yields the blue curve. In this case, we observe an increased amount of failures, indicating the importance of overrepresenting $\trajfailset$, and slightly weaker control performance as $\trajretset$ is not represented.

We provide an analysis of the learned safety filter in Appendix~\ref{subsec:safetyFilterAnalysis} and visualize the \dynasaur{} training process in a video\footnote{Video of training: \url{https://data-science-in-mechanical-engineering.github.io/Dyna-SAuR/}.} in the supplementary material.

\WFclear

\section{Related Work}\label{sec:related_work}
Besides penalizing constraint violations in the reward function \cite{massiani2022safe}, safety in RL is frequently addressed through constrained Markov decision processes (CMDPs) \cite{altmann_cmdp}, where the expected cumulative safety cost is constrained by a budget \cite{wabersich2023data}. While widely adopted \cite{achiam2017constrained, tessler2018reward, Ray2019SafeExploration, sauteRL, PrimalDualSafeRL, PrimalDualSafeRL3_FOCOPS, PrimalDualSafeRL4_P3O, actsafe_ICLR_2025_similarToIsol}, CMDPs address safety in expectation over trajectories, which is insufficient for many applications such as collision avoidance~\cite{zhao2023state}, where even a single constraint violation can be catastrophic.

State-wise constrained MDPs (SCMDPs) make stricter statements by formulating a cost budget per state~\cite{zhao2023state} instead of trajectory wide budgets. Methods addressing state-wise safety of the deployed agent after training often require minimal prior knowledge \cite{AsymptoticStatewise2019, AsymptoticStatewise2020, AsymptoticStatewise2022}. Conversely, approaches addressing training-time safety typically rely on strong assumptions, such as known or stable dynamics and prior safe controllers \cite{SafeRLModelGivenCBF4LyapunovNeeded,StateWiseSafetyDuringTraining2017,StateWiseSafetyDuringTraining2018,StateWiseSafetyDuringTraining2022, Thananjeyan2020RecoveryRS}.

Viability-based approaches yield rigorous safety filters providing hard state-wise safety guarantees, but often face computational limits or demand substantial domain knowledge \cite{wabersich2023data}. Methods based on Hamilton–Jacobi reachability, control barrier functions, or Lyapunov functions \cite{SafeRLModelGivenHJR, StateWiseSafetyDuringTraining2021, SafeRLModelGivenCBF1PartialNominalModel, SafeRLModelGivenCBF2Modelneeded, SafeRLModelGivenCBF3NeedsControlInvarSet} generally require nominal models or a given subset of viable states. 

Data-driven approximations to viability-based safety filters reduce these requirements but face limitations, including the inability to handle discontinuous dynamics~\cite{Wang2022EnforcingHC}, the need for unfiltered environment interaction, or accurate sampling models~\cite{learningCBF,HJRModelFree,black-box-lavankul, dalal2018safeexplorationcontinuousaction}.

To the best of our knowledge, \dynasaur{ } is the first RL method that approximates viability-based filters for nonlinear and unstable systems, without access to the environment dynamics or a subset of viable states. Further, it provides training-time safety for systems with state-action dimensionalities beyond the reach of existing viability-based approaches.

\section{Concluding Remarks}\label{sec:ConclusionAndLimitations}

Safe exploration is a fundamental concern in RL, yet traditional safety filters require substantial domain knowledge and struggle with high-dimensional problems. Conversely, RL excels at learning control for high-dimensional systems that are hard to model from first principles.
We propose \dynasaur{}, a data-driven approach to filter synthesis that jointly trains control and filter policies using a learned uncertainty-aware dynamics model.
Inspired by viability theory, we present a compact RL formulation that approximates a rigorous safety filter matching the predictive capabilities of the model. Specifically, we combine the notion of viability and certainty, introduce a novel filter parametrization, and provide scaling guidelines. Using the model, we limit required domain knowledge to a notion of failure and an initial dataset. \dynasaur{} matches or exceeds established safe RL control performance while reducing failures by at least two orders of magnitude.

Providing a novel perspective on viability-based safety filtering, \dynasaur{} also raises several interesting directions for future research, addressing current limitations of the method. Among these are improving the exploration capabilities in high-dimensional problems, e.g., through active exploration of the dynamics model \cite{buisson2020actively}, improving the compute efficiency by leveraging knowledge of the control and filter policies of prior iterations instead of retraining them from scratch, or extending the underlying theory beyond control-affine systems.
Nonetheless, we believe \dynasaur{} represents a significant step towards safer, more reliable RL agents.

\section*{Acknowledgement}
We thank Lukas Kesper, Devdutt Subhasish, and Pierre-François Massiani for the valuable discussions on the work presented in this paper.
This work was funded in part by the German Federal Ministry of Research, Technology and Space (BMFTR) under grant number 02P22A110. The responsibility for the content of this publication lies with the authors. 
The authors gratefully acknowledge the computing time provided to them at the NHR Center NHR4CES at RWTH Aachen University (project number p0022301). This is funded by the Federal Ministry of Education and Research, and the state governments participating on the basis of the resolutions of the GWK for national high performance computing at universities (www.nhr-verein.de/unsere-partner). We acknowledge EuroHPC Joint Undertaking for awarding us access to MareNostrum5 as BSC, Spain. Friedrich Solowjow is supported by the KI-Starter grant by the state of NRW. This work is funded in part under the Excellence Strategy of the Federal Government and the Länder (G:(DE-82)EXS-SF-OPSF854).  (G:(DE-82)EXS-SF-OPSF854). 

\bibliographystyle{abbrvnat} 
\bibliography{main}
\newpage
\appendix
\onecolumn
\section{Detailed Algorithm Description of \dynasaur{}}
\label{app:pseudocode}

In this section, we provide a more detailed overview of how the three learning problems depicted in \Figref{fig:detailed_block_diagram} interact with each other. We begin with model learning in~\Secref{app:modelLearning}, then continue with the  filter learning problem in~\Secref{app:safety_filter_learning_problem}, and conclude with the (safety-filtered) control learning problem in~\Secref{app:perfTraining}. Algorithm~\ref{alg:dynasaur_app} combines all learning problems into the \dynasaur{ }algorithm.

\begin{algorithm}
\caption{\dynasaur{ }}
\label{alg:dynasaur_app}
\begin{algorithmic}

\INPUT $\buffer^{\moddyn, \dynfcn}_0$, randomly initialized $\moddyn_0$, $\pol_0$ and $\filterpol_0$
\FOR{$\algiter \in \{1, \dots, \maxalgiter \}$ \dynasaur~iterations}
\STATE Train dynamics model $\moddyn_{\algiter}$ using data from $\buffer^{\moddyn, \dynfcn}_{\algiter-1}$ (Algorithm~\ref{alg:model_learning})
\REPEAT
\STATE Train filter policy $\filterpol_{\algiter}$ using $\moddyn_{\algiter}$ and $\pol_{\algiter-1}$\\ (Algorithm~\ref{alg:filter_learning})
\STATE Evaluate filter policy $\filterpol_{\algiter}$ (Algorithm~\ref{alg:filter_eval})
\UNTIL{average evaluation episode length does not improve}
\STATE Train control policy $\pol_\algiter$ using $\moddyn_{\algiter}$ and $\filterpol_{\algiter}$ and collect env. data $\buffer^{\moddyn, \dynfcn}_{\algiter}$ (Algorithm~\ref{alg:performance_learning})
\ENDFOR
\OUTPUT $\moddyn_{\maxalgiter}, \filterpol_{\maxalgiter}, \pol_{\maxalgiter}$
\end{algorithmic}
\end{algorithm}
\begin{figure}
    \centering
    \iflatexml\includegraphics[width=\textwidth]{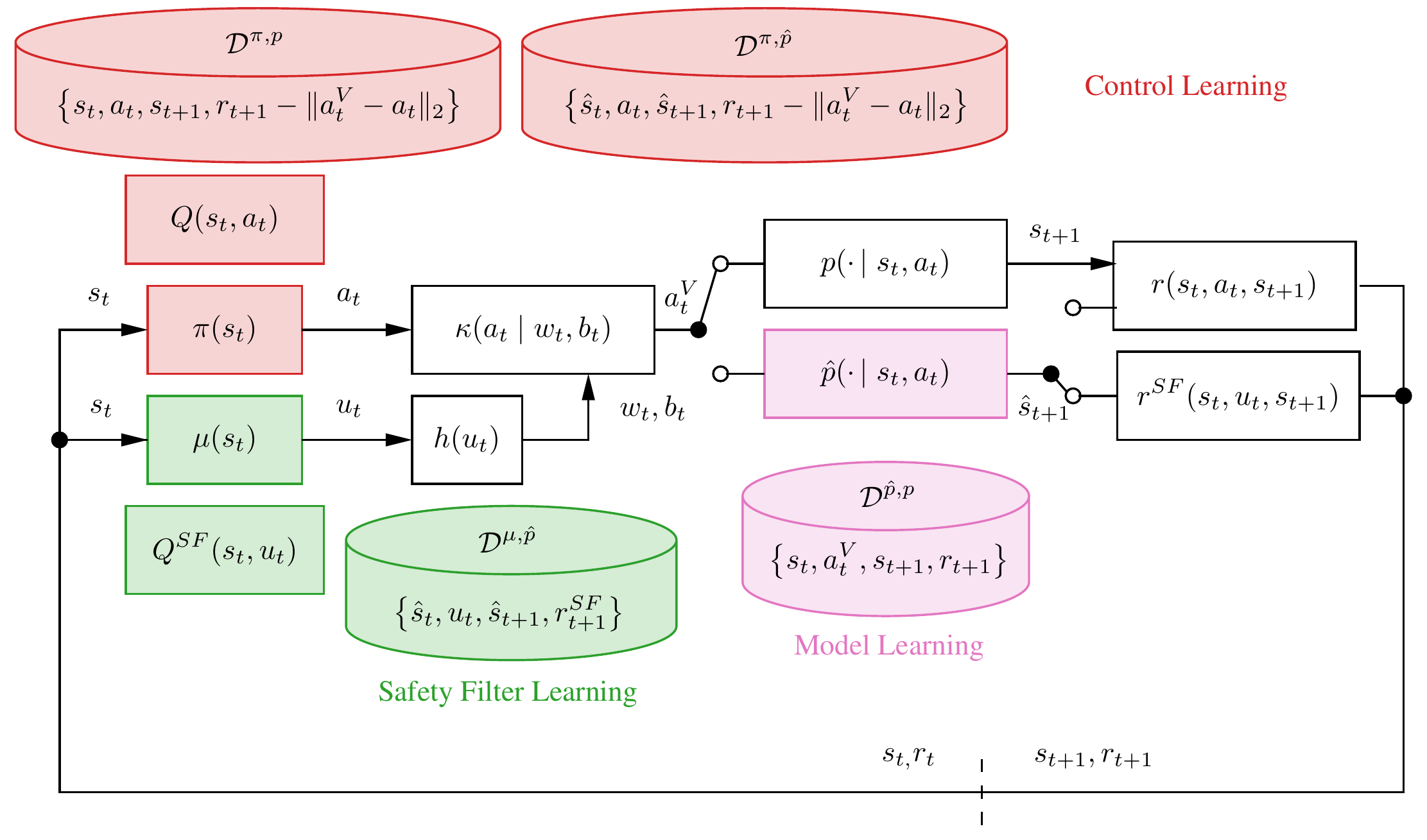}\else\tikzset{every picture/.style={line width=0.75pt}} %
\resizebox{\linewidth}{!}{
\begin{tikzpicture}[x=0.75pt,y=0.75pt,yscale=-1,xscale=1]

\draw  [color={rgb, 255:red, 214; green, 39; blue, 40 }  ,draw opacity=1 ][fill={rgb, 255:red, 214; green, 39; blue, 40 }  ,fill opacity=0.2 ] (240,19.5) -- (240,58.5) .. controls (240,67.06) and (190.75,74) .. (130,74) .. controls (69.25,74) and (20,67.06) .. (20,58.5) -- (20,19.5) .. controls (20,10.94) and (69.25,4) .. (130,4) .. controls (190.75,4) and (240,10.94) .. (240,19.5) .. controls (240,28.06) and (190.75,35) .. (130,35) .. controls (69.25,35) and (20,28.06) .. (20,19.5) ;
\draw  [color={rgb, 255:red, 214; green, 39; blue, 40 }  ,draw opacity=1 ][fill={rgb, 255:red, 214; green, 39; blue, 40 }  ,fill opacity=0.2 ] (470,19.5) -- (470,58.5) .. controls (470,67.06) and (420.75,74) .. (360,74) .. controls (299.25,74) and (250,67.06) .. (250,58.5) -- (250,19.5) .. controls (250,10.94) and (299.25,4) .. (360,4) .. controls (420.75,4) and (470,10.94) .. (470,19.5) .. controls (470,28.06) and (420.75,35) .. (360,35) .. controls (299.25,35) and (250,28.06) .. (250,19.5) ;
\draw  [color={rgb, 255:red, 214; green, 39; blue, 40 }  ,draw opacity=1 ][fill={rgb, 255:red, 214; green, 39; blue, 40 }  ,fill opacity=0.2 ] (70,80) -- (160,80) -- (160,120) -- (70,120) -- cycle ;
\draw  [color={rgb, 255:red, 44; green, 160; blue, 44 }  ,draw opacity=1 ][fill={rgb, 255:red, 44; green, 160; blue, 44 }  ,fill opacity=0.2 ] (70,230) -- (160,230) -- (160,270) -- (70,270) -- cycle ;
\draw  [color={rgb, 255:red, 44; green, 160; blue, 44 }  ,draw opacity=1 ][fill={rgb, 255:red, 44; green, 160; blue, 44 }  ,fill opacity=0.2 ] (320,245.5) -- (320,284.5) .. controls (320,293.06) and (286.42,300) .. (245,300) .. controls (203.58,300) and (170,293.06) .. (170,284.5) -- (170,245.5) .. controls (170,236.94) and (203.58,230) .. (245,230) .. controls (286.42,230) and (320,236.94) .. (320,245.5) .. controls (320,254.06) and (286.42,261) .. (245,261) .. controls (203.58,261) and (170,254.06) .. (170,245.5) ;
\draw  [color={rgb, 255:red, 227; green, 119; blue, 194 }  ,draw opacity=1 ][fill={rgb, 255:red, 227; green, 119; blue, 194 }  ,fill opacity=0.2 ] (480,225.5) -- (480,264.5) .. controls (480,273.06) and (450.9,280) .. (415,280) .. controls (379.1,280) and (350,273.06) .. (350,264.5) -- (350,225.5) .. controls (350,216.94) and (379.1,210) .. (415,210) .. controls (450.9,210) and (480,216.94) .. (480,225.5) .. controls (480,234.06) and (450.9,241) .. (415,241) .. controls (379.1,241) and (350,234.06) .. (350,225.5) ;
\draw   (518.25,110) -- (628.25,110) -- (628.25,150) -- (518.25,150) -- cycle ;
\draw   (520,160) -- (630,160) -- (630,200) -- (520,200) -- cycle ;
\draw  [color={rgb, 255:red, 227; green, 119; blue, 194 }  ,draw opacity=1 ][fill={rgb, 255:red, 227; green, 119; blue, 194 }  ,fill opacity=0.2 ] (360,150) -- (470,150) -- (470,190) -- (360,190) -- cycle ;
\draw   (360,100) -- (470,100) -- (470,140) -- (360,140) -- cycle ;
\draw   (200,180) -- (250,180) -- (250,220) -- (200,220) -- cycle ;
\draw   (200,130) -- (310,130) -- (310,170) -- (200,170) -- cycle ;
\draw  [color={rgb, 255:red, 44; green, 160; blue, 44 }  ,draw opacity=1 ][fill={rgb, 255:red, 44; green, 160; blue, 44 }  ,fill opacity=0.2 ] (80,180) -- (150,180) -- (150,220) -- (80,220) -- cycle ;
\draw  [color={rgb, 255:red, 214; green, 39; blue, 40 }  ,draw opacity=1 ][fill={rgb, 255:red, 214; green, 39; blue, 40 }  ,fill opacity=0.2 ] (80,130) -- (150,130) -- (150,170) -- (80,170) -- cycle ;
\draw    (630,180) -- (650,180) -- (650,360) -- (40,360) -- (40,200) -- (78,200) ;
\draw [shift={(80,200)}, rotate = 180] [fill={rgb, 255:red, 0; green, 0; blue, 0 }  ][line width=0.08]  [draw opacity=0] (12,-3) -- (0,0) -- (12,3) -- cycle    ;
\draw  [dash pattern={on 4.5pt off 4.5pt}]  (458.5,344.9) -- (458.5,357.7) -- (458.5,377.9) ;
\draw    (630,130) -- (650,130) -- (650,360) -- (40,360) -- (40,150) -- (78,150) ;
\draw [shift={(80,150)}, rotate = 180] [fill={rgb, 255:red, 0; green, 0; blue, 0 }  ][line width=0.08]  [draw opacity=0] (12,-3) -- (0,0) -- (12,3) -- cycle    ;
\draw    (150,200) -- (198,200) ;
\draw [shift={(200,200)}, rotate = 180] [fill={rgb, 255:red, 0; green, 0; blue, 0 }  ][line width=0.08]  [draw opacity=0] (12,-3) -- (0,0) -- (12,3) -- cycle    ;
\draw    (150,150) -- (198,150) ;
\draw [shift={(200,150)}, rotate = 180] [fill={rgb, 255:red, 0; green, 0; blue, 0 }  ][line width=0.08]  [draw opacity=0] (12,-3) -- (0,0) -- (12,3) -- cycle    ;
\draw    (250,200) -- (280,200) -- (280,172) ;
\draw [shift={(280,170)}, rotate = 90] [fill={rgb, 255:red, 0; green, 0; blue, 0 }  ][line width=0.08]  [draw opacity=0] (12,-3) -- (0,0) -- (12,3) -- cycle    ;
\draw    (310,150) -- (330,150) ;
\draw [shift={(330,150)}, rotate = 0] [color={rgb, 255:red, 0; green, 0; blue, 0 }  ][fill={rgb, 255:red, 0; green, 0; blue, 0 }  ][line width=0.75]      (0, 0) circle [x radius= 3.35, y radius= 3.35]   ;
\draw    (360,120) -- (342.35,120) ;
\draw [shift={(340,120)}, rotate = 180] [color={rgb, 255:red, 0; green, 0; blue, 0 }  ][line width=0.75]      (0, 0) circle [x radius= 3.35, y radius= 3.35]   ;
\draw    (360,170) -- (342.35,170) ;
\draw [shift={(340,170)}, rotate = 180] [color={rgb, 255:red, 0; green, 0; blue, 0 }  ][line width=0.75]      (0, 0) circle [x radius= 3.35, y radius= 3.35]   ;
\draw    (330,150) -- (338,123) ;
\draw    (520,140) -- (502.35,140) ;
\draw [shift={(500,140)}, rotate = 180] [color={rgb, 255:red, 0; green, 0; blue, 0 }  ][line width=0.75]      (0, 0) circle [x radius= 3.35, y radius= 3.35]   ;
\draw    (520,180) -- (502.35,180) ;
\draw [shift={(500,180)}, rotate = 180] [color={rgb, 255:red, 0; green, 0; blue, 0 }  ][line width=0.75]      (0, 0) circle [x radius= 3.35, y radius= 3.35]   ;
\draw    (470,170) -- (490,170) ;
\draw [shift={(490,170)}, rotate = 0] [color={rgb, 255:red, 0; green, 0; blue, 0 }  ][fill={rgb, 255:red, 0; green, 0; blue, 0 }  ][line width=0.75]      (0, 0) circle [x radius= 3.35, y radius= 3.35]   ;
\draw    (497,178) -- (490,170) ;
\draw    (470,120) -- (518,120) ;
\draw [shift={(520,120)}, rotate = 180] [fill={rgb, 255:red, 0; green, 0; blue, 0 }  ][line width=0.08]  [draw opacity=0] (12,-3) -- (0,0) -- (12,3) -- cycle    ;
\draw    (630,180) -- (650,180) ;
\draw [shift={(650,180)}, rotate = 0] [color={rgb, 255:red, 0; green, 0; blue, 0 }  ][fill={rgb, 255:red, 0; green, 0; blue, 0 }  ][line width=0.75]      (0, 0) circle [x radius= 3.35, y radius= 3.35]   ;
\draw    (80,200) -- (40,200) ;
\draw [shift={(40,200)}, rotate = 180] [color={rgb, 255:red, 0; green, 0; blue, 0 }  ][fill={rgb, 255:red, 0; green, 0; blue, 0 }  ][line width=0.75]      (0, 0) circle [x radius= 3.35, y radius= 3.35]   ;

\draw (130.5,48.5) node   [align=left] {$\displaystyle \left\{s_{t} ,a_{t} ,s_{t+1} ,r_{t+1} -\| a_{t}^{V} -a_{t} \| _{2}\right\}$};
\draw (128.7,19.8) node   [align=left] {$\displaystyle \mathcal{D}^{\pi ,p}$};
\draw (360.5,48.5) node   [align=left] {$\displaystyle \left\{\hat{s}_{t} ,a_{t} ,\hat{s}_{t+1} ,r_{t+1} -\| a_{t}^{V} -a_{t} \| _{2}\right\}$};
\draw (359.7,19.8) node   [align=left] {$\displaystyle \mathcal{D}^{\pi ,\hat{p}}$};
\draw (115,100) node   [align=left] {$\displaystyle Q( s_{t} ,a_{t})$};
\draw (115,250) node   [align=left] {$\displaystyle Q^{SF}( s_{t} ,u_{t})$};
\draw (243.45,245.3) node   [align=left] {$\displaystyle \mathcal{D}^{\mu ,\hat{p}}$};
\draw (244,279) node   [align=left] {$\displaystyle \left\{\hat{s}_{t} ,u_{t} ,\hat{s}_{t+1} ,r_{t+1}^{SF}\right\}$};
\draw (416.5,255) node   [align=left] {$\displaystyle \left\{s_{t} ,a_{t}^{V} ,s_{t+1} ,r_{t+1}\right\}$};
\draw (416,224.5) node   [align=left] {$\displaystyle \mathcal{D}^{\hat{p} ,p}$};
\draw (573.25,130) node   [align=left] {$\displaystyle r( s_{t} ,a_{t} ,s_{t+1})$};
\draw (575,180) node   [align=left] {$\displaystyle r^{SF}( s_{t} ,u_{t} ,s_{t+1})$};
\draw (415,170) node   [align=left] {$\displaystyle \hat{p}( \cdotp \mid s_{t} ,a_{t})$};
\draw (415,120) node   [align=left] {$\displaystyle p( \cdotp \mid s_{t} ,a_{t})$};
\draw (224.85,200.19) node   [align=left] {$\displaystyle h( u_{t})$};
\draw (255.17,150.19) node   [align=left] {$\displaystyle \kappa ( a_{t} \mid w_{t} ,b_{t})$};
\draw (115,200) node   [align=left] {$\displaystyle \mu ( s_{t})$};
\draw (115,150) node   [align=left] {$\displaystyle \pi ( s_{t})$};
\draw (509.48,345) node   [align=left] {$\displaystyle s_{t+1} ,r_{t+1}$};
\draw (425.5,345) node   [align=left] {$\displaystyle s_{t,} r_{t}$};
\draw (59,186) node   [align=left] {$\displaystyle s_{t}$};
\draw (58,135.25) node   [align=left] {$\displaystyle s_{t}$};
\draw (171.5,135.25) node   [align=left] {$\displaystyle a_{t}$};
\draw (171.5,186) node   [align=left] {$\displaystyle u_{t}$};
\draw (309.5,186) node   [align=left] {$\displaystyle w_{t} ,b_{t}$};
\draw (492,107) node   [align=left] {$\displaystyle s_{t+1}$};
\draw (487,186) node   [align=left] {$\displaystyle \hat{s}_{t+1}$};
\draw (322.5,135.25) node   [align=left] {$\displaystyle a_{t}^{V}$};
\draw (416.5,294.5) node  [color={rgb, 255:red, 227; green, 119; blue, 194 }  ,opacity=1 ] [align=left] {Model Learning};
\draw (243.5,315.5) node  [color={rgb, 255:red, 227; green, 119; blue, 194 }  ,opacity=1 ] [align=left] {\textcolor[rgb]{0.17,0.63,0.17}{Safety Filter Learning}};
\draw (551.5,40.5) node  [color={rgb, 255:red, 227; green, 119; blue, 194 }  ,opacity=1 ] [align=left] {\textcolor[rgb]{0.84,0.15,0.16}{Control Learning}};

\end{tikzpicture}
} \fi
    \caption{Overview of the three learning problems in \dynasaur{} and their interactions.
    \emph{During filter learning, the control policy $\pi$ is fixed. The filter replay buffer $\mathcal{D}^{\mu,\hat{p}}$ is populated with model transitions $(\hat{s}_t, u_t, \hat{s}_{t+1}, r^{\mathrm{SF}}_{t+1})$, where the filter policy $\mu$ transforms potentially unsafe actions $a_t$ into safe actions $a_t^V$ while exploring the action space. During control learning, the filter $\mu$ is fixed. The control policy $\pi$ maintains two replay buffers: an environment control buffer $\mathcal{D}^{\pi,p}$ and a model control buffer $\mathcal{D}^{\pi,\hat{p}}$. Model rollouts, in which actions proposed by $\pi$ are filtered by $\mu$, are used to populate $\mathcal{D}^{\pi,\hat{p}}$. In parallel, the control policy interacts with the environment, with its actions filtered by $\mu$, to populate $\mathcal{D}^{\pi,p}$. In both cases, the environment reward is penalized by the magnitude of filtering, $\lVert a_t^V - a_t \rVert_2$. The unfiltered action $a_t$ is stored in both control buffers. Transitions from the environment control buffer are additionally used to populate the model learning buffer $\mathcal{D}^{\hat{p},p}$, without reward penalization and using the filtered action $a_t^V$, as this is the action applied to the true dynamics $p$.}}
    \label{fig:detailed_block_diagram} 
\end{figure}
\newpage
\subsection{Probabilistic Ensemble Model Learning}
\label{app:modelLearning}

\begin{algorithm}
\caption{Dynamics Model Learning}
\label{alg:model_learning}
\scriptsize
\begin{algorithmic}
\INPUT $\buffer^{\moddyn, \dynfcn}_{\algiter-1}, \moddyn_\modparams$
\FOR{model training steps}
\STATE Sample mini-batch of $N^{\mathrm{Model}}$ transitions from $\buffer^{\moddyn, \dynfcn}$
\STATE Sample ensemble member $\moddyn_{\modparams_e}$ with $e \sim \unidist(\{1, \dots, E \})$
\STATE $\modparams_e \leftarrow \arg\min_{\modparams_e} - \frac{1}{N^{\mathrm{Model}}} \sum \log \moddyn_{\modparams_e}(\state_{t+1}|\state_t, \viaact_t)$
\ENDFOR
\OUTPUT $\moddyn_\modparams$
\end{algorithmic}
\end{algorithm}

At the beginning of iteration $j$ of \dynasaur{}, we use the model learning buffer $\mathcal{D}^{\hat{p},p}_{j-1}$ to train a probabilistic ensemble model~\cite{Lakshminarayanan2017Dec, chua_deep_2018} by minimizing the negative log-likelihood, as described in \Algref{alg:model_learning}.
Probabilistic ensembles (PE) have members each predicting a normal distribution over the next state. The idea is that ensemble members converge to identical predictions within the data support, while the members' predictions disagree in out-of-distribution settings. This disagreement increases the predictive uncertainty of the overall PE model in the form of epistemic uncertainty.

Recent work on probabilistic ensemble models~\cite{Frauenknecht2025Infoprop} restricts model rollouts to a sufficiently certain set $\accurset{j} \subseteq (\statesp \times \actsp)$ in which the environment dynamics are modeled accurately. Infoprop~\cite{Frauenknecht2025Infoprop} separates aleatoric from epistemic
uncertainty in model rollouts. At each step, it draws a realization of the
aleatoric noise process, and the remaining uncertainty about this realization,
which stems from epistemic uncertainty, is quantified by the entropy
$\entropy(\modrvstate_{t+1})$ of the Infoprop state $\modrvstate_{t+1}$\footnote{We use $\modrvstate_{t+1}$ for the Infoprop state $\tilde{S}_{t+1}$ in the Infoprop paper~\cite{Frauenknecht2025Infoprop}.}. Based on this quantity, Infoprop restricts model
rollouts to a sufficiently certain set $\accurset{j} \subseteq \statesp \times \actsp$,
in which the model is assumed to be sufficiently accurate:
\begin{equation}
\accurset{j} := \{(\state_t, \act_t) \in \statesp \times \actsp \mid
\entropy(\modrvstate_{t+1}) \leq \singleStepThreshold\},
\label{eq:accur_set_app}
\end{equation}
i.e., the entropy of the
predicted next Infoprop state is at most $\singleStepThreshold$. The model of
iteration~$j$ is trained on $\mathcal{D}^{\hat{p},p}_{j-1}$, which contains
data collected within $\accurset{j-1}$.

In addition, the accumulated information loss is thresholded for each trajectory $\hat{\tau}$
\begin{equation}
\hat{\tau}=(\modstate_0,\act_0, \modstate_1, \ldots, \act_{T-1}, \modstate_T)
\end{equation}
by $\multiStepThreshold$. This defines sufficiently accurate trajectories of length $t' \in \{1, \ldots, T\}$ as
\begin{equation}\label{eq:accumThreshold}
\left\{ (\modstate_t, \act_t)_{t=0}^{t'-1} \in \accurset{j}^{t'} \midd \sum_{t=0}^{t'-1} \informationLoss(\modrvstate_{t+1}) \leq \multiStepThreshold \right\},
\end{equation}
which we use throughout the paper when referring to model rollouts. The authors of~\cite{Frauenknecht2025Infoprop} provide heuristics for determining suitable values of $\singleStepThreshold$ and $\multiStepThreshold$ based on the training data of the current model.

We consider leaving $\accurset{j}$ a failure, forcing the safety filter to stay in the sufficiently certain set. When the accumulation threshold is triggered, we truncate the rollout since longer rollouts are unreliable but we do not consider this a failure.

We proceed with filter learning in the current \dynasaur{} iteration.

\subsection{Filter Learning Problem}\label{app:safety_filter_learning_problem}
\subsubsection{Filter Policy Learning}
\begin{algorithm}
\caption{Filter Policy Learning}
\label{alg:filter_learning}
\scriptsize
\begin{algorithmic}
\INPUT $\moddyn_{\algiter}, \pol_{\algiter -1}, \filterpol_{\filterpolparams}, \filterpol_{\bar{\filterpolparams}}, \filterqfcn_{\filterqfcnparams_1}, \filterqfcn_{\bar{\filterqfcnparams}_1}, \filterqfcn_{\filterqfcnparams_2}, \filterqfcn_{\bar{\filterqfcnparams}_2}, \buffer^{\filterpol, \moddyn} = \emptyset$ 
\STATE Randomly initialize $\filterpolparams, \filterqfcnparams_1, \filterqfcnparams_2$ and $\bar{\filterpolparams} \leftarrow \filterpolparams, \bar{\filterqfcnparams}_1 \leftarrow \filterqfcnparams_1, \bar{\filterqfcnparams}_2 \leftarrow \filterqfcnparams_2$
    \FOR{$\safetyFilterModelBasedRollouts$ model-based rollouts}
    \STATE $\expscale_1 \sim \unidist(\expscale_{\mathrm{min}}, \expscale_{\mathrm{max}})$ \ALGCOMMENT{\scriptsize{Sample exploration scale \cite{fasttd3}}}
    \STATE $\behpol(\state_t)$ is either $\pol_{\algiter -1}(\state_t) + \expscale_2 \cdot \noise^{\behpol}_t$ or $\expscale_3 \cdot \noise^{\behpol}_t$ with $\noise^{\behpol}_t \sim \pinknoiseproc^{\behpol}$ \ALGCOMMENT{\scriptsize{Choose between stochastic performance policy and pink noise process as behavior policy}}
    \STATE $\modstate_0 \sim \initdistsf$ \ALGCOMMENT{\scriptsize{Sample initial state from modified starting state distribution introduced in \Secref{subsec:safe_dist_formulation}}}
        \WHILE{$t < T$ and $\entropy(\rvmodstate_t) \leq \enttresh_1$ and $\sum_{t^{\prime} = 0}^t\entropy(\rvmodstate_{t^{\prime}}) \leq \enttresh_2 $ and $\modstate_t \in \safeset$}
        \STATE $\act_t \sim \behpol(\modstate_t)$ \ALGCOMMENT{\scriptsize{Sample potentially unsafe action from behavior policy}}
        \STATE $\asf_t \sim \filterpol_{\filterpolparams}(\modstate_t) + \expscale_1 \cdot \noise_t^{\filterpol} $ with $\noise_t^{\filterpol} \sim \pinknoiseproc^{\filterpol}$ \ALGCOMMENT{\scriptsize{Sample hyperplane action exploring with pink noise}}
        \STATE $\asf_t \leftarrow \asf_t \,/\, \max\!\left(1, \|\asf_t\|_2\right)$ \ALGCOMMENT{\scriptsize{Clip to unit Euclidean ball}}
        \STATE $\viaact_t = \filter(\act_t, \hyperplaneVector_t, \hyperplaneOffset_t)$ with $(\hyperplaneVector_t, \hyperplaneOffset_t) = \hyperplaneTransformationFunction(\asf_t)$ \ALGCOMMENT{\scriptsize{Potentially filter behavior policy according to~\eqref{eq:learned_filter}}}
        \STATE $\modstate_{t+1} \sim \moddyn(\cdot | \modstate_t, \viaact_t)$ \ALGCOMMENT{\scriptsize{Sample next state via the Infoprop rollout~\cite{Frauenknecht2025Infoprop}}}
        \STATE $\rewsf_{t+1} = \rewfcnsf(\modstate_t, \asf_t)$ \ALGCOMMENT{\scriptsize{Get next reward according to \eqref{eq:reward_func_safe}}}
        \STATE $\buffer^{\filterpol, \moddyn} \leftarrow \buffer^{\filterpol, \moddyn} \cup \{( \modstate_t, \asf_t, \modstate_{t+1}, \rewsf_{t+1})\}$ \ALGCOMMENT{\scriptsize{Store transition in filter replay buffer}}
        \ENDWHILE
    \ENDFOR
    \FOR{update steps}
    \STATE Sample mini-batch of $N$ transitions from $\buffer^{\filterpol, \moddyn}$
    \STATE $y = \sum_{k = t}^{t+n-1} \discount^{k-t} \rewsf_{k+1} + \discount^{n} \min_{i \in \{1, 2 \}}\filterqfcn_{\bar{\filterqfcnparams}_i}(\modstate_{t+n}, \filterpol_{\bar{\filterpolparams}}(\modstate_{t+n}) + \expscale_4 \cdot \noise)$ with $\noise \sim \mathcal{N}(0, I)$  \ALGCOMMENT{\scriptsize{Compute the $n$-step TD target $y$}}
    \STATE $\filterqfcnparams_i \leftarrow \arg\min_{\filterqfcnparams_i} \frac{1}{N}\sum(y-\filterqfcn_{\filterqfcnparams_i}(\modstate_t, \asf_t))^2$ for $i \in \{1, 2\}$\ALGCOMMENT{\scriptsize{Update the Q-functions by minimizing the squared TD error}}
        \IF{policy update frequency}
        \STATE $\filterpolparams \leftarrow \arg\max_{\filterpolparams} \frac{1}{N}\sum (\filterqfcn_{\filterqfcnparams_1}(\modstate_t, \filterpol_{\filterpolparams}(\modstate_t)) - c \|\filterpol_{\filterpolparams}(\modstate_t)\|_2)$ \ALGCOMMENT{\scriptsize{Update policy by maximizing Q-value while penalizing excessive filtering \eqref{eq:rl_filter_learning_objective}}}
        \STATE $\bar{\filterqfcnparams}_i \leftarrow \polyak \filterqfcnparams_i + (1-\polyak) \bar{\filterqfcnparams}_i$ for $i \in \{1, 2\}$ \ALGCOMMENT{\scriptsize{Polyak averaging}}
        \STATE $\bar{\filterpolparams} \leftarrow \polyak \filterpolparams + (1-\polyak) \bar{\filterpolparams}$ \ALGCOMMENT{\scriptsize{Polyak averaging}}
        \ENDIF
    \ENDFOR
\OUTPUT $\filterpol_{\filterpolparams}$
\end{algorithmic}
\end{algorithm}

To learn the filter $\filterpol_{j}$ in iteration $j$ of \dynasaur{}, we require the current model $\hat{p}_j$ and the control policy $\pi_{j-1}$ from the previous iteration. The learning procedure for the filter is described in \Algref{alg:filter_learning}. We adapt the recent model-free RL algorithm FastTD3~\cite{fasttd3}, which has demonstrated strong performance by fully exploiting parallelization in TD3. FastTD3 employs massively parallelized environments, large batch sizes, $n$-step returns, multiple exploration noise scales, and a distributional critic.

Since the distributional critic improves sample efficiency at the cost of increased wall-clock runtime, we adopt all components of FastTD3 except the distributional critic. This design choice is motivated by the fact that sample complexity is not critical in our setting, as we rely on parallel rollouts from copies of the learned model rather than parallel simulation of the real environment.

Each filter learning phase is initialized from scratch, using newly initialized Q-functions $\filterqfcn_{\filterqfcnparams_1}$ and $\filterqfcn_{\filterqfcnparams_2}$, as well as a newly initialized policy $\filterpol_{\filterpolparams}$, together with their corresponding target networks $\filterqfcn_{\bar{\filterqfcnparams}_1}$, $\filterqfcn_{\bar{\filterqfcnparams}_2}$, and $\filterpol_{\bar{\filterpolparams}}$. The filter replay buffer $\mathcal{D}^{\mu,\hat{p}}$ is also reset at the beginning of each learning phase.

In parallel, we perform $\safetyFilterModelBasedRollouts$ parallel model rollouts. For each rollout, a noise scale is sampled uniformly as $\expscale_1 \sim \unidist(\expscale_{\mathrm{min}}, \expscale_{\mathrm{max}})$~\cite{fasttd3}. For each rollout, the behavior policy is chosen with equal probability to be either a pink-noise process or the control policy augmented with pink-noise exploration. Rollouts are terminated if $(\modstate_{t},\viaact_t) \notin \safesetSFi{j-1}$. Rollouts are truncated if the maximum episode length is reached or if the accumulated information loss is over the threshold as in \eqref{eq:accumThreshold}.

Following the diagram in \Figref{fig:detailed_block_diagram} and~\Algref{alg:filter_learning}, in each model rollout a potentially unsafe action $a_t$ proposed by the behavior policy is sampled to evaluate the exploratory hyperplane action $u_t$. This action is filtered by solving the quadratic program defined in~\eqref{eq:learned_filter}. The resulting filtered action is then applied to the model via an Infoprop rollout~\cite{Frauenknecht2025Infoprop} to sample the next state $\modstate_{t+1}$.

The reward function provides a positive signal when the 
filter agent keeps the state within the intersection of the safe and certain sets, and otherwise penalizes the agent according to~\eqref{eq:reward_func_safe}. The resulting transition is stored in the filter replay buffer. After each step across all model rollouts, the Q-functions and policy are trained using FastTD3. A large mini-batch is sampled to compute the $n$-step TD target $y$, which is then used to update the Q-functions by minimizing the squared TD error. Delayed policy updates are performed, where we introduce the action regularization term in~\eqref{eq:rl_filter_learning_objective} to penalize excessive filtering. As in FastTD3, Polyak averaging is used for updating the target networks.

After a fixed number of steps we start the evaluation of the current filter $\mu_j$.

\subsubsection{Filter Evaluation}
\begin{algorithm}
\caption{Filter Evaluation}
\label{alg:filter_eval}
\scriptsize
\begin{algorithmic}
\INPUT $\moddyn_{\algiter}, \pol_{\algiter -1},\filterpol_{\algiter}$ 
    \FOR{number of evaluations}
    \STATE $\behpol(\state_t)$ is either $\pol_{\algiter -1}(\state_t) + \expscale_2 \cdot \noise^{\behpol}_t$ or $\expscale_3 \cdot \noise^{\behpol}_t$ with $\noise^{\behpol}_t \sim \pinknoiseproc^{\behpol}$ \ALGCOMMENT{\scriptsize{Choose between stochastic performance policy and pink noise process as behavior policy}}
    \STATE $\modstate_0 \sim \initdistsfeval(\cdot)$ \ALGCOMMENT{\scriptsize{Sample initial state}}
        \WHILE{$t < T$ and $\entropy(\rvmodstate_t) \leq \enttresh_1$ and $\sum_{t^{\prime} = 0}^t\entropy(\rvmodstate_{t^{\prime}}) \leq \enttresh_2 $ and $\modstate_t \in \safeset$}
        \STATE $\act_t \sim \behpol(\modstate_t)$ \ALGCOMMENT{\scriptsize{Sample potentially unsafe action from behavior policy}}
        \STATE $\asf_t = \filterpol_{\algiter}(\modstate_t)$ \ALGCOMMENT{\scriptsize{Sample deterministic hyperplane action}}
        \STATE $\viaact_t = \filter(\act_t, \hyperplaneVector_t, \hyperplaneOffset_t)$ with $(\hyperplaneVector_t, \hyperplaneOffset_t) = \hyperplaneTransformationFunction(\asf_t)$ \ALGCOMMENT{\scriptsize{Potentially filter behavior policy according to~\eqref{eq:learned_filter}}}
        \STATE $\modstate_{t+1} \sim \moddyn(\cdot | \modstate_t, \viaact_t)$ \ALGCOMMENT{\scriptsize{Sample next state via the Infoprop rollout~\cite{Frauenknecht2025Infoprop}}}
        \ENDWHILE
    \ENDFOR
\OUTPUT $\trajfailset, \trajretset$ \ALGCOMMENT{\scriptsize{Update starting states}}
\end{algorithmic}
\end{algorithm}

The goal of the filter evaluation is to determine if the RL filter training is finished. To this end, we perform model rollouts using the greedy filter policy $\mu$ and evaluate how long the behavior policy, defined as in the previous section, can be kept safe. Each evaluation model rollout is initialized from a starting state sampled from $\initdistsfeval(\cdot)$, the evaluation starting-state distribution, which is defined as a mixture of the control starting-state distribution $\initdist$ and states from $\trajretset$.

We compute the average episode length across evaluation rollouts and terminate filter training if this metric does not improve compared to previous evaluations. We adopt this heuristic to make multiple RL filter training runs feasible within a reasonable time frame.

In addition to deciding when to end filter training and proceed to control learning, we update the filter starting-state distribution $\initdistsf$ after each evaluation. Parts of each terminated trajectory are added to the failure set $\trajfailset$, while parts of each truncated trajectory are added to $\trajretset$.
We utilize the concept of maximum time to failure, which describes the maximum number of steps for which a safe state in the unviability kernel can be kept safe before it inevitably becomes unsafe. For most systems, this quantity is infeasible to compute exactly. We therefore design a heuristic that approximates the maximum time to failure using an interval defined by two hyperparameters, $[\mathrm{MTF}_{\mathrm{early}}, \mathrm{MTF}_{\mathrm{late}}]$. This interval is intended to capture the range in which the true maximum time to failure is expected to lie.

Given a terminated trajectory $\trajfail$ with terminal time $T_{\mathrm{f}}$, we sample $\mathrm{MTF}_{\mathrm{num}}$ states from
\begin{equation}
(\modstate_{T_{\mathrm{f}}-\mathrm{MTF}_{\mathrm{late}}}, \dots, \modstate_{T_{\mathrm{f}}-\mathrm{MTF}_{\mathrm{early}}})
\end{equation}
and add them to $\trajfailset$ for each failed rollout. The motivation behind this heuristic is to sample training starting states in the vicinity of failure regions of the  filter that can still be corrected. By sampling in this manner, we heuristically select states from the viability kernel.

From the truncated trajectories $\trajret$, we aim to overrepresent those that yield high returns with respect to the control reward $r$, as an overly conservative safety filter would be particularly harmful in these regions for the subsequent control policy learning problem. At the same time, to ensure that selected starting states are viable, we apply a heuristic that only selects high-return trajectories whose Q-values, used here as proxies for the expected time to failure, exceed a predefined threshold.

After a variable number of training and evaluation rounds, the method proceeds to control policy learning.

\subsection{Control Learning Problem}\label{app:perfTraining}
\begin{algorithm}
\caption{Control Policy Learning and Environment Data Collection}
\label{alg:performance_learning}
\scriptsize
\begin{algorithmic}
\INPUT $\dynfcn, \moddyn_{\algiter}, \filterpol_{\algiter}, \pol_{\polparams}, \pol_{\bar{\polparams}}, \qfcn_{\qfcnparams_1}, \qfcn_{\bar{\qfcnparams}_1}, \qfcn_{\qfcnparams_2}, \qfcn_{\bar{\qfcnparams}_2}, \buffer^{\pol, \dynfcn} = \emptyset,   \buffer^{\pol, \moddyn} = \emptyset, \buffer^{\moddyn, \dynfcn}_{\algiter-1}$
\STATE Randomly initialize $\polparams, \qfcnparams_1, \qfcnparams_2$ and $\bar{\polparams} \leftarrow \polparams, \bar{\qfcnparams}_1 \leftarrow \qfcnparams_1, \bar{\qfcnparams}_2 \leftarrow \qfcnparams_2$
\STATE $\state_0 \sim \initdist$
    \FOR{environment steps}
    \STATE $\act_t \sim \pol_\polparams(\state_t) + \expscale_2 \cdot \noise^{\behpol}_t$ with $\noise^{\behpol}_t \sim \pinknoiseproc^{\behpol}$ \ALGCOMMENT{\scriptsize{Sample exploratory control action}}
    \STATE $\asf_t = \filterpol_{\algiter}(\state_t)$ \ALGCOMMENT{\scriptsize{Sample filter policy action}}
    \STATE $\viaact_t = \filter(\act_t, \hyperplaneVector_t, \hyperplaneOffset_t)$ with $(\hyperplaneVector_t, \hyperplaneOffset_t) = \hyperplaneTransformationFunction(\asf_t)$\ALGCOMMENT{\scriptsize{Filter control action}}
    \STATE $\state_{t+1} \sim \dynfcn(\cdot | \state_t, \viaact_t)$ \ALGCOMMENT{\scriptsize{Sample next state of environment by executing filtered action}}
    \STATE $\rew_{t+1} = \rewfcn(\state_t, \viaact_t, \state_{t+1})$ \ALGCOMMENT{\scriptsize{Determine control reward}}
    \STATE $\tilde{\rew}_{t+1} = \rew_{t+1} - \| \viaact_t - \act_t \|_2$ \ALGCOMMENT{\scriptsize{Penalize filtering}}
    \STATE $\buffer^{\moddyn, \dynfcn} \leftarrow \buffer^{\moddyn, \dynfcn} \cup \{( \state_t, \viaact_t, \state_{t+1}, \rew_{t+1})\}$ \ALGCOMMENT{\scriptsize{Add filtered action to model learning buffer}}
    \STATE $\buffer^{\pol, \dynfcn} \leftarrow \buffer^{\pol, \dynfcn} \cup \{ (\state_t, \act_t, \state_{t+1}, \tilde{\rew}_{t+1} )\}$ \ALGCOMMENT{\scriptsize{Add unfiltered action to environment control buffer}}
        \IF{model-based rollout frequency}
            \FOR{number of model-based rollouts}
            \STATE $\expscale_1 \sim \unidist(\expscale_{\mathrm{min}}, \expscale_{\mathrm{max}})$ \ALGCOMMENT{\scriptsize{Sample exploration scale~\cite{fasttd3}}}
            \STATE $\modstate_0 \sim \buffer^{\pol, \dynfcn}$
                \WHILE{$t < T$ and $\entropy(\rvmodstate_t) \leq \enttresh_1$ and $\sum_{t^{\prime} = 0}^t\entropy(\rvmodstate_{t^{\prime}}) \leq \enttresh_2 $}
                \STATE $\act_t \sim \pol_\polparams(\modstate_t) + \expscale_1 \cdot \noise^{\pol}_t$ with $\noise^{\pol}_t \sim \pinknoiseproc^{\pol}$\ALGCOMMENT{\scriptsize{Sample exploratory control action}}
                \STATE $\asf_t = \filterpol_{\algiter}(\modstate_t)$ \ALGCOMMENT{\scriptsize{Sample deterministic hyperplane action}}
                \STATE $\viaact_t = \filter(\act_t, \hyperplaneVector_t, \hyperplaneOffset_t)$ with $(\hyperplaneVector_t, \hyperplaneOffset_t) = \hyperplaneTransformationFunction(\asf_t)$ \ALGCOMMENT{\scriptsize{Potentially filter behavior policy according to~\eqref{eq:learned_filter}}}
                \STATE $\modstate_{t+1} \sim \moddyn(\cdot | \modstate_t, \viaact_t)$ \ALGCOMMENT{\scriptsize{Sample next state via the Infoprop rollout~\cite{Frauenknecht2025Infoprop}}}
                \STATE $\rew_{t+1} = \rewfcn(\modstate_t, \viaact_t, \modstate_{t+1})$ \ALGCOMMENT{\scriptsize{Determine control reward}}
    \STATE $\tilde{\rew}_{t+1} = \rew_{t+1} - \| \viaact_t - \act_t \|_2$ \ALGCOMMENT{\scriptsize{Penalize filtering}}
                \STATE $\buffer^{\pol, \moddyn} \leftarrow \buffer^{\pol, \moddyn} \cup \{( \modstate_t, \act_t, \modstate_{t+1}, \tilde{\rew}_{t+1} )\}$ \ALGCOMMENT{\scriptsize{Add unfiltered action to model control buffer}}
                \ENDWHILE
            \ENDFOR
        \ENDIF
        \FOR{update steps}
        \STATE Sample mini-batch of $N$ transitions from $\buffer^{\pol, \moddyn} \cup \buffer^{\pol, \dynfcn}$
        \STATE $y = \tilde{\rew}_{t+1} + \discount \min_{i \in \{1, 2 \}}\qfcn_{\bar{\qfcnparams}_i}(\modstate_{t+1}, \pol_{\bar{\polparams}}(\modstate_{t+1}) + \expscale_4 \cdot \noise)$  with $\noise \sim \mathcal{N}(0, I)$ \ALGCOMMENT{\scriptsize{Compute the TD target $y$}}
        \STATE $\qfcnparams_i \leftarrow \arg\min_{\qfcnparams_i} \frac{1}{N}\sum(y-\qfcn_{\qfcnparams_i}(\modstate_t, \act_t))^2$ for $i \in \{1, 2\}$\ALGCOMMENT{\scriptsize{Update the Q-functions by minimizing the squared TD error}}
            \IF{policy update frequency}
            \STATE $\polparams \leftarrow \arg\max_{\polparams} \frac{1}{N}\sum \qfcn_{\qfcnparams_1}(\modstate_t, \pol_{\polparams}(\modstate_t))$ \ALGCOMMENT{\scriptsize{Update policy by maximizing Q-value}}
            \STATE $\bar{\qfcnparams}_i \leftarrow \polyak \qfcnparams_i + (1-\polyak) \bar{\qfcnparams}_i$ for $i \in \{1, 2\}$ \ALGCOMMENT{\scriptsize{Polyak averaging}}
            \STATE $\bar{\polparams} \leftarrow \polyak \polparams + (1-\polyak) \bar{\polparams}$ \ALGCOMMENT{\scriptsize{Polyak averaging}}
            \ENDIF
        \ENDFOR
    \ENDFOR
\OUTPUT $\pol_\polparams, \buffer^{\moddyn, \dynfcn}_{\algiter}$
\end{algorithmic}
\end{algorithm}
At the start of each control policy RL phase, we are given the filter $\mu_j$ and the model $\hat{p}_{j}$ from the current iteration. The training procedure for the control policy is described in Algorithm~\ref{alg:performance_learning}. Essentially, this procedure follows the Infoprop algorithm, with FastTD3 serving as the model-free RL backbone. Additionally, we incorporate recent insights from~\cite{safetyFilterWhileTraining, markgraf2025SafeReinforcementLearning} by penalizing the control agent for deviations from the safety-filtered action, as illustrated in \Figref{fig:detailed_block_diagram}. The penalized reward is given by $\tilde{r}_{t+1}=r_{t+1} - \lVert a_t - a_t^V \rVert_2$, where the Euclidean distance measures the deviation between the control agent’s proposed action and the safety-filtered action.

When applying this penalty, it is important to store the unfiltered action $a_t$ in both the environment control buffer $\mathcal{D}^{\pi,p}$ and the model control buffer $\mathcal{D}^{\pi,\hat{p}}$. The filtered action $a_t^V$, which is executed in the environment, is stored in the model learning buffer $\mathcal{D}^{\hat{p},p}$. Training is terminated after a fixed number of environment interactions. The \dynasaur{ }training loop continues with the next iteration of model learning.

\clearpage
\newpage
\section{Proofs}\label{sec:proofsAppendix}
\subsection{Bijection Between Parameter Space and Discriminating Hyperplanes}\label{app:proofTransformation}
\begin{lemma}[Maximum Offset of a Hyperplane Intersecting the Interior of a Hypercube] 
\label{lemma:MaxOffset}

Let $w \in \mathbb{R}^{\na}$ be a normalized normal vector, i.e., $\|w\|_2 = 1$.  
A hyperplane $\{a \in \mathbb{R}^{\na} \mid w^\top a = b\}$ intersects the interior of a hypercube $\actsp = [-1,1]^{\na}$ if and only if
\[
|b| < \|w\|_1.
\]
\end{lemma}
\begin{proof}
A hyperplane intersects the interior $(-1,1)^{\na}$ if and only if there exists $a \in (-1,1)^{\na}$ with $w^\top a = b$, which holds if and only if
\[
b \in \left(\inf_{a \in (-1,1)^{\na}} w^\top a,\ \sup_{a \in (-1,1)^{\na}} w^\top a\right).
\]
Since $w^\top a$ is linear and $w \neq 0$, the image of the convex open set $(-1,1)^{\na}$ is an open interval, whose infimum and supremum equal the minimum and maximum over the closed hypercube. Therefore:
\[
\sup_{a \in (-1,1)^{\na}} w^\top a = \max_{a \in [-1,1]^{\na}} w^\top a = \sum_{i=1}^{\na} |w_i| = \|w\|_1.
\]
By symmetry, the infimum is $-\|w\|_1$. The hyperplane intersects the interior if and only if $b \in (-\|w\|_1, \|w\|_1)$, i.e.,
\[
|b| < \|w\|_1.
\]
\end{proof}

\begin{theorem}[Bijection]
\label{theo:bijection_app}
The function $\hyperplaneTransformationFunction$ mapping each $\asf \in \actspsf$ to a discriminating hyperplane $(w,b)$ with $\Vert w \Vert_2 = 1$ intersecting the interior of $\actsp = [-1, 1]^{\na}$ is bijective. It is defined as $\hyperplaneTransformationFunction(\asf_t) = (\hyperplaneVector_t,\hyperplaneOffset_t)$ with
\begin{equation*}
    \hyperplaneVector_t = \frac{\asf_t}{\Vert \asf_t \Vert_2}, \text{ and } \hyperplaneOffset_t = (2\cdot\Vert \asf_t \Vert_2 -1 ) \Vert  \hyperplaneVector_t \Vert _1.
\end{equation*}
\end{theorem}

\begin{proof}

\emph{Well-definedness.} For any $\asf \in \actspsf=\{u \in \mathbb{R}^{\na} \mid 0 <\Vert u\Vert_2 < 1\}$, the hyperplane
$\{a \in \mathbb{R}^{\na} \mid \hyperplaneVector^\top a = \hyperplaneOffset\}$, with $h(u)=(w,b)$ intersects the interior of the action space
$\actsp = [-1,1]^{\na}$.

By construction, $\hyperplaneVector$ is a normalized normal vector with
$\|\hyperplaneVector\|_2 = 1$. Since $0 < \|\asf\|_2 < 1$, we have
\[
2\|\asf\|_2 - 1 \in (-1,1),
\]
and therefore
\[
|\hyperplaneOffset|
= |2\|\asf\|_2 - 1| \, \|\hyperplaneVector\|_1
< \|\hyperplaneVector\|_1.
\]
By Lemma~\ref{lemma:MaxOffset}, these conditions are sufficient for the hyperplane to intersect $(-1,1)^{\na}$.
Given $\asf$, the construction uniquely determines both the normalized normal vector
$\hyperplaneVector$ and the offset $\hyperplaneOffset$. 

\emph{Injectivity.} Suppose $\hyperplaneTransformationFunction(u) = \hyperplaneTransformationFunction(v)$ for $u, v \in \actspsf$, i.e., $(w_u, b_u) = (w_v, b_v)$. From the equality of normal vectors:
\[
\frac{u}{\|u\|_2} = \frac{v}{\|v\|_2},
\]
it follows that $v = \lambda u$ for some $\lambda > 0$. From the equality of offsets and $\|w_u\|_1 = \|w_v\|_1 > 0$:
\[
(2\|u\|_2 - 1)\|w_u\|_1 = (2\|v\|_2 - 1)\|w_v\|_1 \implies \|u\|_2 = \|v\|_2.
\]
Combined with $v = \lambda u$, this gives $\lambda = 1$, hence $u = v$. Thus $\hyperplaneTransformationFunction$ is injective.

\emph{Surjectivity.}
Let $(w,b)$ be a hyperplane intersecting $(-1,1)^{\na}$ with normalized normal vector, i.e., $\|w\|_2 = 1$. By Lemma~\ref{lemma:MaxOffset} the absolute offset must be strictly smaller than the maximal offset $|b| < \|w\|_1$. Define
\[
l = \frac{1}{2}\left(\frac{b}{\|w\|_1} + 1\right),
\qquad
u = w \, l.
\]
Then $0<l=\|u\|_2 < 1$ and $u \in \actspsf$. Substituting $u$ into
$\hyperplaneTransformationFunction$ recovers $(w,b)$ by construction. Hence, $\hyperplaneTransformationFunction$ is surjective.

Therefore, $\hyperplaneTransformationFunction$ is bijective.
\end{proof}

\subsection{Expected Time to Failure}\label{app:time_to_failure}
\newcommand{\ret}{G^{\mathrm{SF}}}
\newcommand{\ft}{T_{\mathrm{F}}}

\begin{lemma}[Expected Time to Failure]
\label{lemma:time_fail}
The filter action value function $\filterqfcn(\state, \asf)$ of a filter policy $\filterpol(\state)$, with $\discountsf \in (0,1)$, for a reward function
\begin{equation*}
\rewfcnsf(\state_t,\asf_t) =
\begin{cases}
1, & \text{if } (\state_{t}, \viaact_t) \in \safesetSFi{}, \\
-\frac{1}{1-\discountsf}, & \text{otherwise}
\end{cases}
\end{equation*}
lower bounds the expected time to failure $\ft$ under the filter policy $\filterpol$ in the following form:
\begin{equation*}
    \mathbb{E}_\filterpol \lb\ft|\state \rb \geq \log_{\discountsf} \left( \frac{1 - \filterqfcn(\state, \filterpol(\state)) (1 - \discountsf)}{2}\right)
\end{equation*}
\end{lemma}

\begin{proof}
The safety filter return $\ret$ for a trajectory terminating after $\ft$ steps is given by
\begin{equation*}
    \begin{aligned}
         \ret &= \sum_{t=0}^{\ft -1} \discountsf^t \cdot 1 + \discountsf^{\ft} \left( - \frac{1}{1-\discountsf}\right)\\
        &= \frac{1 - \discountsf^{\ft}}{1 - \discountsf} - \frac{\discountsf^{\ft}}{1-\discountsf}\\
        & =\frac{1 - 2 \discountsf^{\ft}}{1-\discountsf}.
    \end{aligned}
\end{equation*}
The state value function $V^{\mathrm{SF}}(\state) = \filterqfcn(\state, \filterpol(\state))$ \cite{Sutton1998} encodes the expected return under the policy $\filterpol$
\begin{equation}
\label{eq:ttf}
    \begin{aligned}
        \filterqfcn(\state, \filterpol(\state)) &= \mathbb{E}_\filterpol \lb \ret \mid \rvstate =\state, U = \filterpol(\state)\rb\\
        &= V^{\mathrm{SF}}(\state)\\
        &= \mathbb{E}_\filterpol \lb \ret \mid \state \rb\\
        &= \mathbb{E}_\filterpol \lb \frac{1 - 2 \discountsf^{\ft}}{1-\discountsf} \Big|\state \rb\\
        &\stackrel{\mathrm{a}}{\leq} \frac{1 - 2 \discountsf^{\mathbb{E}_\filterpol \lb\ft|\state \rb}}{1-\discountsf}
    \end{aligned}
\end{equation}
\begin{minipage}{\linewidth}
    \footnotesize $^{\mathrm{a}}$\,By Jensen's inequality applied to the concave function
\end{minipage}
Rearranging \eqref{eq:ttf} yields the desired result.
\end{proof}

\clearpage
\newpage

\section{Additional Results}
\subsection{Ablation}\label{app:ablationWalker}
\begin{figure}
    \centering
    \includegraphics[width=0.6\linewidth]{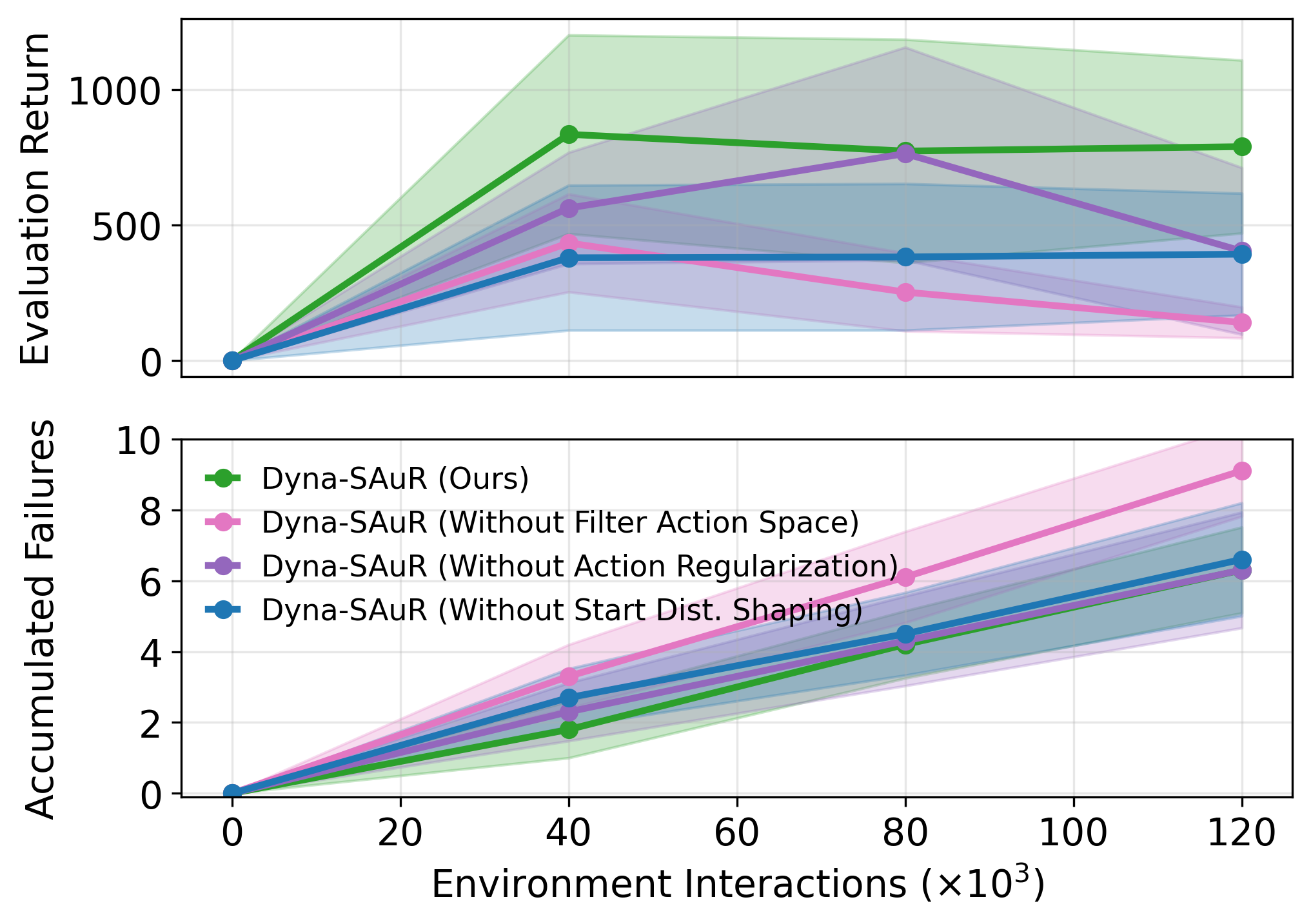}
    \caption{Ablation of \dynasaur{} design choices for Walker. \emph{Removing the starting state distribution shaping introduced in Section~\ref{subsec:safe_dist_formulation} or the action regularization introduced in Section~\ref{subsec:safe_rew_formulation} leads to reduced control performance and increased accumulated failures. In contrast, incorrect parameterization of the hyperplane defining action introduced in Section~\ref{subsec:safe_act_formulation} prevents effective learning, indicating its central role in the method.}}
    \label{fig:ablationWalker} 
\end{figure}

\begin{figure}
    \centering
    \includegraphics[width=0.6\linewidth]{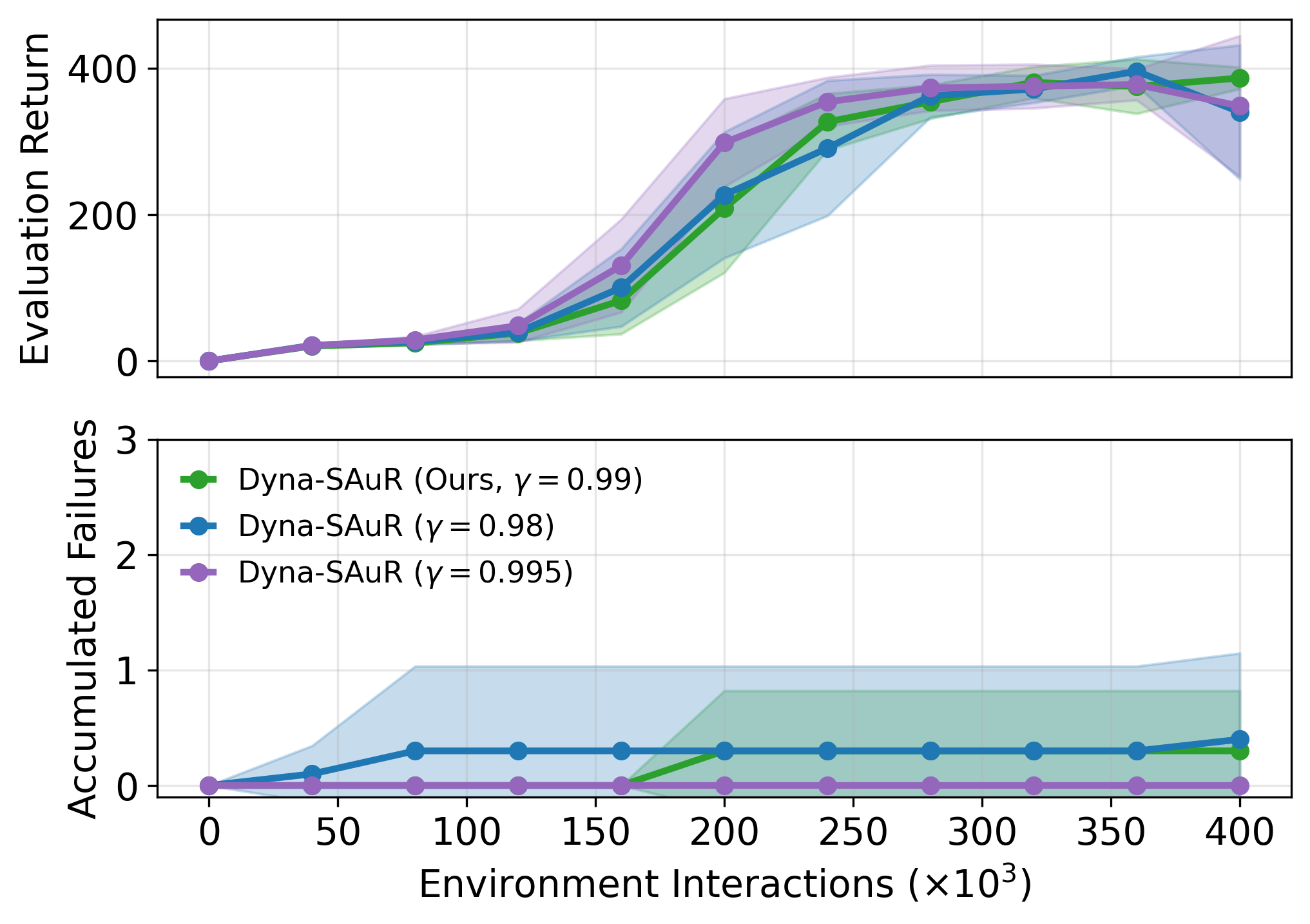}

    \caption{Ablation of the discount factor $\gamma$ of \dynasaur{} on CartPole. \emph{The main results use $\gamma = 0.99$ without tuning. Decreasing ($\gamma = 0.98$) or increasing ($\gamma = 0.995$) the discount factor has no significant effect on the evaluation return or the accumulated failures.}
    \label{fig:ablationGamma} }
\end{figure}

As visualized in \Figref{fig:ablationWalker}, the Walker ablation exhibits trends similar to those observed for CartPole with \dynasaur{} depicted in \Figref{fig:ablation_cartpole}. When removing the efficient action-space parameterization of the safety filter, we observe the highest number of accumulated failures and the lowest final return. This indicates that the proposed parameterization reduces the search space and improves the effectiveness of learning safety filters with RL. When the action regularization is removed, the safety filter RL agent tends to learn overly restrictive safety filters, which leads to a lower final return. When the high-return starting states as well as the failure-preventing starting states are removed, i.e., without starting state distribution shaping, the final return is also reduced.

The main results of the paper use a discount factor of $\gamma = 0.99$, which we did not tune. To assess its influence, we ran \dynasaur{} on CartPole with $\gamma \in \{0.98, 0.99, 0.995\}$ for 10 seeds each, with otherwise identical code and hyperparameters. The $\gamma = 0.99$ runs of this ablation were repeated with the same code version as the other values and therefore differ slightly from the main results.
As shown in Figure~\ref{fig:ablationGamma}, changing $\gamma$ has no significant effect on the final return in CartPole or the accumulated safety violations.

\subsection{Safety Filter Analysis}\label{subsec:safetyFilterAnalysis}
\begin{wrapfigure}{r}{0.45\textwidth}
    \centering
    \includegraphics[width=1\linewidth]{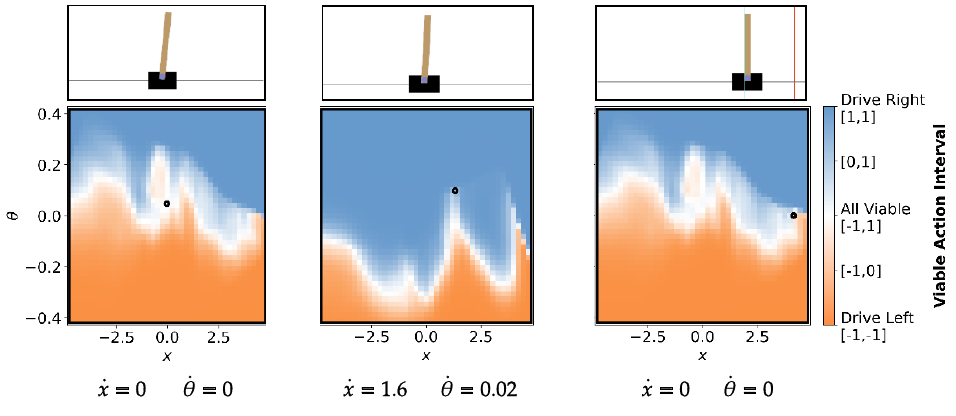}
    \caption{Final safety filter $\filterpol_\maxalgiter$ in goal-reaching CartPole. \emph{
     The filter is less restrictive around the upper equilibrium and for low velocities, indicated by the comparatively large white areas, and becomes more restrictive as deflections and velocities increase.}}
    \label{fig:finalfilter}
\end{wrapfigure}
Finally, we observe the safety filter $\filterpol_\maxalgiter$ after the final training iteration $\maxalgiter$ for goal-reaching CartPole in Figure~\ref{fig:finalfilter}.\footnote{Video of the filter learning process: \url{https://data-science-in-mechanical-engineering.github.io/Dyna-SAuR/}} The top row shows the CartPole at the beginning, in the middle, and at the end of a trajectory in the environment. The bottom row shows $\filterpol_\maxalgiter$ in the $(x,\theta)$ subspace, with velocities given below each frame specifying the current state indicated by the black dot. Filtering the 1-dimensional action, namely applying a force in the interval $[-1, 1]$ to the cart, is achieved by reducing the interval. This is visualized via color-coding, where blue and orange indicate that only force to the right or left is considered viable, respectively, while white indicates that all actions are considered viable. We observe a clear correlation between the pole angle and the set of viable actions, which aligns with the task dynamics. Further, we see a broad set of viable actions around the upright position at low velocities, which shrinks as deflections and velocities increase.

\subsection{Model Accuracy}\label{subsec:modelAccuracy}
\begin{figure}
    \centering
    \includegraphics[width=\linewidth]{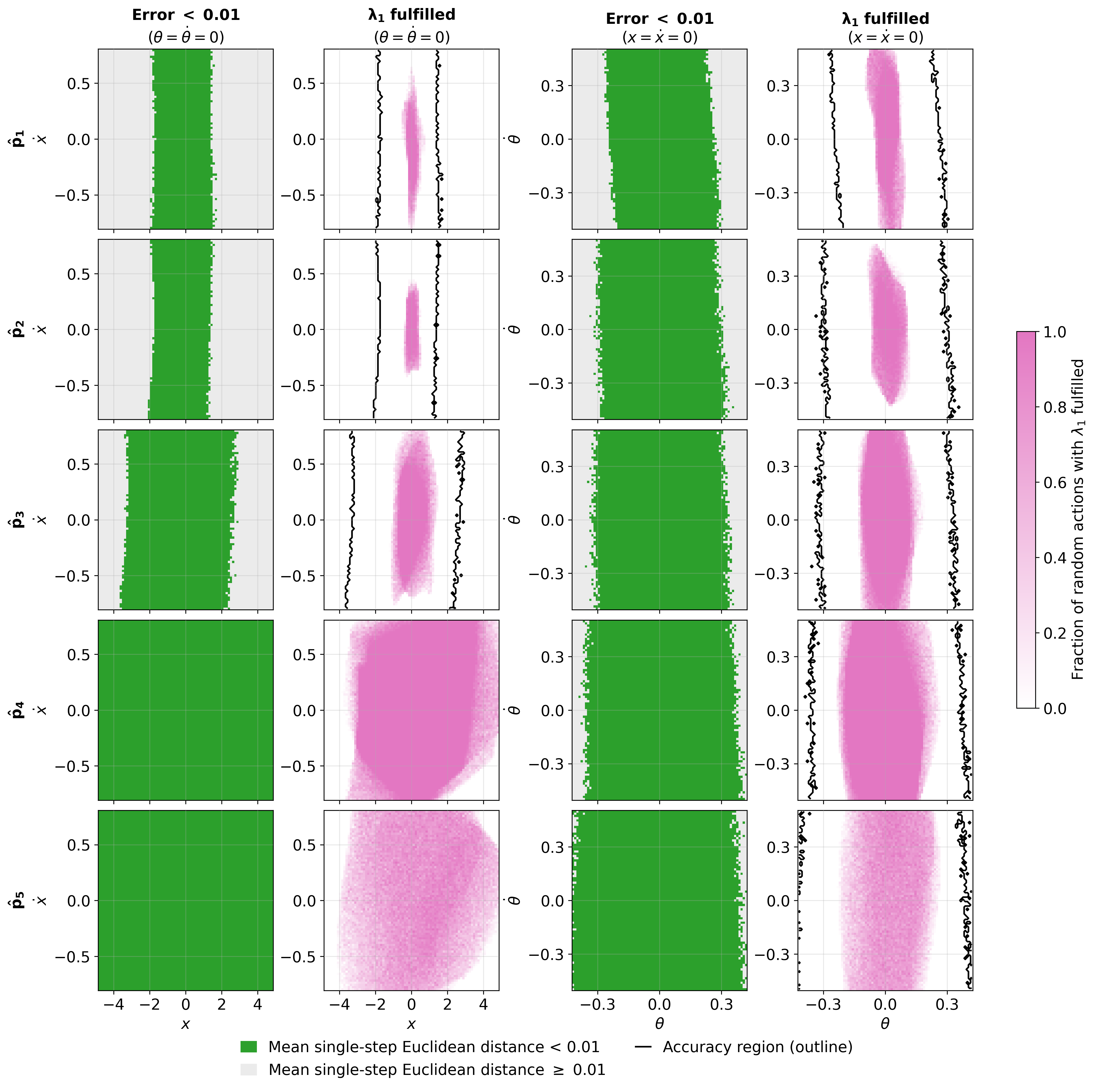}
    \caption{Single-step accuracy of the \dynasaur{} dynamics models $\moddyn_1$ to $\moddyn_5$ (rows) on goal-reaching CartPole, evaluated on two slices of the state space: cart position and velocity with the pole upright at rest (left) and pole angle and angular velocity with the cart at rest (right). Green regions mark states where the Euclidean distance between the predicted and the true next state, averaged over $32$ uniformly sampled actions, is below $0.01$. The black line outlines this region. Pink shows the fraction of these actions for which $(\state, \act) \in \accurset{j}$. In all iterations, the sufficiently certain set $\accurset{j}$ lies within the accurately modeled region, i.e., as long as model rollouts stay in $\accurset{j}$, the environment dynamics are approximated accurately.}
    \label{fig:modelAccuracy}
\end{figure}
In Figure~\ref{fig:modelAccuracy}, we investigate whether the sufficiently certain set $\accurset{j}$ identifies the region in which the learned dynamics model approximates the environment accurately. We consider the dynamics models $\moddyn_1$ to $\moddyn_5$ learned in the first five iterations of a \dynasaur{} training run on goal-reaching CartPole.

Since the state space of CartPole is four-dimensional, we evaluate two two-dimensional slices, each on a regular grid of $81 \times 81$ states: cart position $x$ and velocity $\dot{x}$ with the pole upright at rest ($\theta = \dot{\theta} = 0$), and pole angle $\theta$ and angular velocity $\dot{\theta}$ with the cart at rest in the center ($x = \dot{x} = 0$). The slices extend to the failure thresholds in $x$ and $\theta$ and to the $99.5$\% quantile of the velocities observed during training. In every grid state, we apply the same $32$ actions, sampled uniformly from the action space, to both the environment $\dynfcn$ and the model $\moddyn_j$. For each state-action pair, we determine whether it lies in $\accurset{j}$, with $\singleStepThreshold$ chosen as during training, and compute the Euclidean distance between the true next state and the mean prediction of $\moddyn_j$. 

Within $\accurset{j}$, the environment dynamics are approximated accurately by all models. Every evaluated state-action pair in $\accurset{j}$ has a single-step error below $0.01$. At the same time, the accurately modeled region is considerably larger than $\accurset{j}$, so $\accurset{j}$ is a conservative estimate of where the model can be trusted. As additional data is collected, $\accurset{j}$ expands. Restricting model rollouts to $\accurset{j}$ thus ensures that the filter and control policies are trained on accurate model data, while the growth of $\accurset{j}$ allows for an increasingly less restrictive filter, consistent with the rising returns in Figure~\ref{fig:resultsReturnFailureDynasaur}.

\clearpage
\newpage
\section{Implementation Details}\label{app:experimental_setup}
\subsection{Initial Data}
\label{app:priorData}
Since safety without any prior knowledge is an ill-defined problem~\cite{actsafe_ICLR_2025_similarToIsol}, we assume the availability of some initial data to train an initial dynamics model. In our experiments, this initial data is generated using a stabilizing controller. Each action produced by the prior controller is perturbed with pink noise and, at every second step, it is replaced by a uniformly sampled action to improve action-space coverage and enable reliable uncertainty estimation. 

For CartPole, we design a standard equilibrium-stabilizing Linear Quadratic Regulator (LQR), which is available in our codebase (\url{https://github.com/Data-Science-in-Mechanical-Engineering/Dyna-SAuR}). For Walker, we train an RL agent to step forward and maintain balance. We use a partially trained version of this agent to generate initial data. Videos of the controllers used to collect initial data are provided in the supplementary material.

The offline data collection phase is explicitly decoupled from the \dynasaur{ }algorithm. Initial data is collected off-policy without safety constraints and may be replaced by logged data from any behavior policy. During safe learning, \dynasaur{ }does not utilize the controllers, which have generated the data.

\dynasaur{ }requires sufficient initial data to obtain a locally accurate dynamics model, which is necessary to derive an initial set of safe policies induced by the learned safety filter.
For CartPole, 30,000 environment interactions suffice due to the one-dimensional action space.
For Walker, 4,000,000 interactions are required, as the hyperplane-based safety filter must certify model accuracy over halfspaces of the six-dimensional action space at each state. This necessitates dense data coverage of the action space, but only locally in the state space.  Despite this data requirement, as shown in Fig.~\ref{fig:dataDistEnvs}, \dynasaur{ }generalizes beyond the support of the initial data distribution and learns a walking gait absent from the offline dataset.

\subsection{Safety Constraints}
\paragraph{CartPole.}
The state is $s=(x,\dot{x},\theta,\dot{\theta})^\top \in \mathcal{S}$, where $x$ is the cart position and $\theta$ is the pole angle.
A state is considered a failure state if either the cart leaves the track bounds or the pole exceeds the allowed angular range. Formally, the failure set is defined as:
\[
\failureset^{\mathrm{CartPole}}
=\Bigl\{\, s\in\mathcal{S}\ \Big|\ x<-4.8\ \lor\ x>4.8\ \lor\ \theta<-0.42\ \lor\ \theta>0.42 \Bigr\}.
\]
The task objective is to reach the goal position $x=4.1$, which necessitates fast and agile cart motion. Consequently, the agent is incentivized to operate near the boundary of the admissible pole angle and to approach the track bound aggressively before decelerating within a narrow positional margin. This renders the goal-reaching variant strictly more demanding than the standard balancing task: the agent must simultaneously maximize locomotion speed and satisfy tight state-wise safety constraints, creating an inherent tension between task performance and constraint satisfaction throughout the trajectory.

\paragraph{Walker.}
A state $s\in\mathcal{S}$ is considered a failure state if the Walker’s vertical position $z$ (measured as the $z$-coordinate of the rendered black marker shown in Fig.~\iflatexml\ref{fig:benchmarked_envs}b\else\ref{fig:walker_safety}\fi) leaves the admissible interval, or if the torso angle $\theta$ exceeds its permitted range. Formally, the failure set is defined as:
\[
\failureset^{\mathrm{Walker}}
=\Bigl\{\, s\in\mathcal{S}\ \Big|\ z<0.8\ \lor\ z>1.2\ \lor\ \theta<-0.9\ \lor\ \theta>0.9 \Bigr\}.
\]
We use the original task objective.
However, the upper bound on $z$ restricts the admissible jumping height relative to the unconstrained Walker task, thereby reducing the achievable cumulative return compared to the standard benchmark. Crucially, the constraint induces a non-trivial temporal credit assignment problem. Once the agent is airborne, no corrective action can prevent a constraint violation that results from a prior locomotion decision. Safe behavior therefore requires anticipatory control. The agent must regulate its vertical momentum sufficiently early to avoid the failure set $\failureset^{\mathrm{Walker}}$ throughout the trajectory.
\subsection{Hyperparameters}\label{app:hyperparameter}
All experiments were conducted over $10$ independent random seeds, with hyperparameters chosen to ensure a fair comparison across methods. For all environments, all episodes are truncated after $500$ steps.

\subsubsection{Hyperparameters \dynasaur{}}

\begin{table}[t]
\centering
\caption{Hyperparameters for \dynasaur{} on Walker and CartPole environments.}
\label{tab:dynasaur_hyperparams}
\begin{tabular}{lcc}
\hline
\textbf{Hyperparameter} & \textbf{CartPole} & \textbf{Walker} \\
\hline
\textbf{Model Learning} \\
\hline
Ensemble size $E$ & 7 & 7 \\
Number of hidden layers & 4 & 4 \\
Number of hidden neurons & 200 & 200 \\
Learning rate & 0.0006 & 0.0006 \\
Weight decay & 0.0007 & 0.0007 \\
Patience for early stopping & 10 & 10 \\
Batch size $N^{\mathrm{Model}}$ & 256 & 256 \\
\hline
\textbf{Safety Filter Learning} \\
\hline
Number of hidden layers actor and critic & 2 & 2 \\
Number of hidden neurons actor and critic & 256 & 256 \\
Number of model rollouts $\safetyFilterModelBasedRollouts$ & 100 & 500 \\
Minimum noise scale $\expscale_{\mathrm{min}}$ & 0.001 & 0.001 \\
Maximum noise scale $\expscale_{\mathrm{max}}$ & 0.3 & 0.2 \\
Accurate quantile $\zeta_1$ & 0.99 & 0.99 \\
Exceptionally accurate quantile $\zeta_2$ & 0.01 & 0.01 \\
Learning rate & 0.0003 & 0.0003 \\
Buffer size filter buffer $\buffer^{\filterpol, \moddyn}$ &1,000,000 & 1,000,000\\
Batch size $N$ &100,000&100,000\\
Polyak averaging factor $\bar{\tau}$ & 0.001 & 0.001\\
Policy update frequency &2 & 2\\
Action regularization factor $c$ & 0.1 & 1 \\
Pink noise process scale $\sigma_3$ & 0.33 & 0.33 \\
Policy exploration scale $\sigma_2$ & 0.1 & 0.1 \\
TD3 smoothing scale $\sigma_4$ &0.003 &0.003 \\
Discount Factor $\gamma^{\mathrm{SF}}$ & 0.99 & 0.99 \\
Factor control starting states $\nu_1$ &0.05&0.1 \\
Factor failed rollouts $\nu_2$ &0.3 &0.45\\
Factor high return rollouts $\nu_3$ & 0.65 & 0.45 \\
Mixing factor $\nu_4$ & 0.5 & 0.5 \\
Mixing factor $\nu_5$ &0.5&0.5\\
Number of evaluation rollouts & 2000 & 1000 \\
Maximum time to failure early $\mathrm{MTF}_{\mathrm{early}}$ &10 &20\\
Maximum time to failure late $\mathrm{MTF}_{\mathrm{late}}$ &50 &100\\
Number of states $\mathrm{MTF}_{\mathrm{num}}$ & 10&5 \\
\hline
\textbf{Control Policy Learning} \\
\hline
Number of hidden layers actor and critic & 2 & 2 \\
Number of hidden neurons actor and critic & 256 & 256 \\
Number of hidden layers & 4 & 4 \\
Minimum noise scale $\expscale_{\mathrm{min}}$ & 0.001 & 0.001 \\
Maximum noise scale $\expscale_{\mathrm{max}}$ & 0.2 & 0.2 \\
Accurate quantile $\zeta_1$ & 0.99 & 0.99 \\
Exceptionally accurate quantile $\zeta_2$ & 0.01 & 0.01 \\
Learning rate & 0.0003 & 0.0003 \\
Polyak averaging factor $\bar{\tau}$ & 0.001 & 0.001\\
Policy update frequency &2 & 2\\
Policy exploration scale $\sigma_2$ & 0.1 & 0.1 \\
TD3 smoothing scale $\sigma_4$ &0.003 &0.003 \\
Discount Factor $\gamma$ & 0.99 & 0.99 \\
Number of environment interactions per \dynasaur{ }iteration & 40,000 & 40,000 \\
\hline
\end{tabular}
\end{table}
The hyperparameters for \dynasaur{ } are shown in Table~\ref{tab:dynasaur_hyperparams}.

\subsubsection{Hyperparameters of the Infoprop Baseline}

\begin{table}[t]
\centering
\caption{Hyperparameters for Infoprop on Walker and CartPole environments.}
\label{tab:infoprop_hyperparams}
\begin{tabular}{lcc}
\hline
\textbf{Hyperparameter} & \textbf{CartPole} & \textbf{Walker} \\
\hline
\textbf{Model Learning} \\
\hline
Ensemble size $E$ & 7 & 7 \\
Number of hidden layers & 4 & 4 \\
Number of hidden neurons & 200 & 400 \\
Learning rate & 0.001 & 0.0006 \\
Weight decay & 0.0002 & 0.0007 \\
Patience for early stopping & 9 & 9 \\
Retrain interval & 40,000 & 40,000 \\
\hline
\textbf{Model Rollouts} \\
\hline
Accurate quantile $\zeta_1$ & 0.99 & 0.99 \\
Exceptionally accurate quantile $\zeta_2$ & 0.01 & 0.01 \\
Scaling factor $\xi$ & 1 & 1 \\
Rollout interval & 250 & 250 \\
Rollout batch size & 100,000 & 100,000 \\
\hline
\textbf{SAC Agent} \\
\hline
Number of hidden neurons & 1,024 & 1,024 \\
Number of hidden layers &2 &2 \\
Learning rate & 0.005 & 0.0002\\
SAC target entropy & 0 & -7 \\
Target update interval & 5 & 4 \\
Update steps $G$ & 10 & 10 \\
\hline
\end{tabular}
\end{table}
As shown in Table~\ref{tab:infoprop_hyperparams}, we use the standard Infoprop~\cite{Frauenknecht2025Infoprop} hyperparameters for the baseline. The only exception is that we increase the retraining interval (measured in environment steps) at which the model is retrained. This change is made to keep model training computationally feasible due to the relatively large amount of initial data in the Walker environment, and to match the retraining interval used by the \dynasaur{ }model.

\subsubsection{Hyperparameters PPO-Lagrangian}

\begin{table}[t]
\centering
\caption{Hyperparameters for PPO-Lagrangian on Walker and CartPole environments.}
\label{tab:ppo_lagrangian_hyperparams}
\begin{tabular}{lcc}
\hline
\textbf{Hyperparameter} & \textbf{CartPole} & \textbf{Walker} \\
\hline
Number of hidden neurons & 64 & 128 \\
Number of hidden layers &2 &2 \\
Steps per epoch & 400 & 4000 \\
Gamma & 0.99 & 0.99 \\
Lambda & 0.97 & 0.97 \\
Cost gamma & 0.97 & 0.97 \\
Cost lambda & 0.97 & 0.97 \\
Penalty learning rate & 0.05 &0.05 \\
Target Kullback-Leibler divergence & 0.01 & 0.01 \\
Value function learning rate & 0.001 & 0.001 \\
\hline
\end{tabular}
\end{table}
As shown in Table~\ref{tab:ppo_lagrangian_hyperparams}, we use standard PPO-Lagrangian~\cite{Ray2019SafeExploration} hyperparameters for the benchmark.

\subsubsection{Hyperparameters DH-RL}

\begin{table}[t]
\centering
\caption{Hyperparameters for DH-RL on Walker and CartPole environments.}
\label{tab:rldh_lagrangian_hyperparams}
\begin{tabular}{lcc}
\hline
\textbf{Hyperparameter} & \textbf{CartPole} & \textbf{Walker} \\
\hline
\textbf{Safety Filter PPO Agent} \\
\hline
Number of hidden neurons & 64 & 128 \\
Number of hidden layers &2 &2 \\
Steps per epoch & 400 & 4000 \\
Gamma & 0.99 & 0.99 \\
Lambda & 0.97 & 0.97 \\
Target Kullback-Leibler divergence & 0.01 & 0.01 \\
Value function learning rate & 0.001 & 0.001 \\
Policy learning rate & 0.0003 & 0.0003 \\
Environment Steps for Safety Filter Training & 30000 & 4000000 \\
Clip Ratio & 0.2 & 0.2 \\
\hline
\textbf{Control PPO Agent} \\
Number of hidden neurons & 64 & 128 \\
Number of hidden layers &2 &2 \\
Steps per epoch & 400 & 4000 \\
Gamma & 0.99 & 0.99 \\
Lambda & 0.97 & 0.97 \\
Target Kullback-Leibler divergence & 0.01 & 0.01 \\
Value function learning rate & 0.001 & 0.001 \\
Policy learning rate & 0.0003 & 0.0003 \\
Clip Ratio & 0.2 & 0.2 \\
\hline
\end{tabular}
\end{table}
As shown in Table~\ref{tab:rldh_lagrangian_hyperparams}, we use standard DH-RL~\cite{black-box-lavankul} hyperparameters for the benchmark.

\subsection{Computation Resources}\label{app:Computation Resources}

Parts of the code were adapted from the Infoprop implementation~\cite{Frauenknecht2025Infoprop}, available at \url{https://github.com/data-science-in-mechanical-engineering/infoprop}, published under the MIT license. Each run was conducted on a single NVIDIA H100 GPU with 8 CPU cores at 2.1 GHz and 128 GB of RAM.
Each experimental configuration was evaluated across 10 random seeds and 4 algorithm variants (1 primary and 3 ablations). Experiments on the Walker environment required approximately 24 GPU hours per seed-variant pair, and those on CartPole approximately 8 GPU hours, yielding a total computational budget of 1280 GPU hours for the final experiments reported in this work. A single filter training run takes about 40 minutes for CartPole and about 7 hours for Walker. The baselines required the following training times: PPO-Lag took 3 h for CartPole and 26 h for Walker; DH-RL took 7 h for CartPole and 10 h for Walker; and Infoprop took 14 h for CartPole and 17 h for Walker.

Including all development and benchmark runs, the full budget amounted to 9791 GPU hours. The source code is released under the MIT license.
\clearpage

\clearpage
\newpage
\section{Definitions and Terminology}
\subsection{Discussion of the Filter Policy Learning Problem}\label{subsec:clari_filter_prob}
We write the filter policy learning problem
\begin{equation}
    \filterpol^* = \argmax_{\filterpol \in \filterpolset} \{ \mathrm{Img}(\filter_{\filterpol}^\pol) \subseteq \modelSafePolsSet{} \},
\end{equation}
as a shorthand notation for finding a filter policy $\filterpol^*\in \filterpolset$ such that
\begin{align}
 \mathrm{Img}(\filter_{\filterpol^*}^\pol) \subseteq \modelSafePolsSet{} \text{ and } \nexists\filterpol\in\filterpolset \text { with } \mathrm{Img}(\filter_{\filterpol^*}^\pol) \subsetneq \mathrm{Img}(\filter_{\filterpol}^\pol).
\end{align}

\subsection{Safety and Viability}\label{subsec:safe_and_viable}

The term ``safe states'' is used inconsistently across the safety literature. For instance, the authors of~\cite{AsymptoticStatewise2019} employ it to denote what our work terms viable states, whereas the authors of~\cite{safeLearningInRobotics} use the term consistently with the definition adopted in this work.

Figure~\ref{fig:viablity_vs_safety} illustrates the terminology adopted in this work. As engineers, we define unsafe states $\failureset$ as states the agent must never reach, e.g., states in which the robot has fallen and lies on the ground. Such failure states are typically easy to define. 
The set of safe states $\safeset$ is the complement of the set of unsafe states $\failureset$.
Consequently, some safe states inevitably lead to failure, e.g., a robot that stumbles such that it cannot recover but does not yet lie on the ground.
Therefore, the more informative distinction is between viable states $\viabilitykernel$, from which the system can remain in the set of safe states indefinitely under a suitable control policy, and unviable states $\unviabilitykernel$, which are either already unsafe or from which failure is inevitable. Distinguishing viable from unviable states is the non-trivial problem that \dynasaur{} addresses.

Figure \ref{fig:viablity_vs_safety} provides a minimal illustrative example of the distinction.
The car on the slope is safe but unviable, as it enters $\failureset$ under any control sequence, since the engine torque is insufficient to climb the slope.
\begin{figure}[h]
    \centering
    \iflatexml\includegraphics[width=\textwidth]{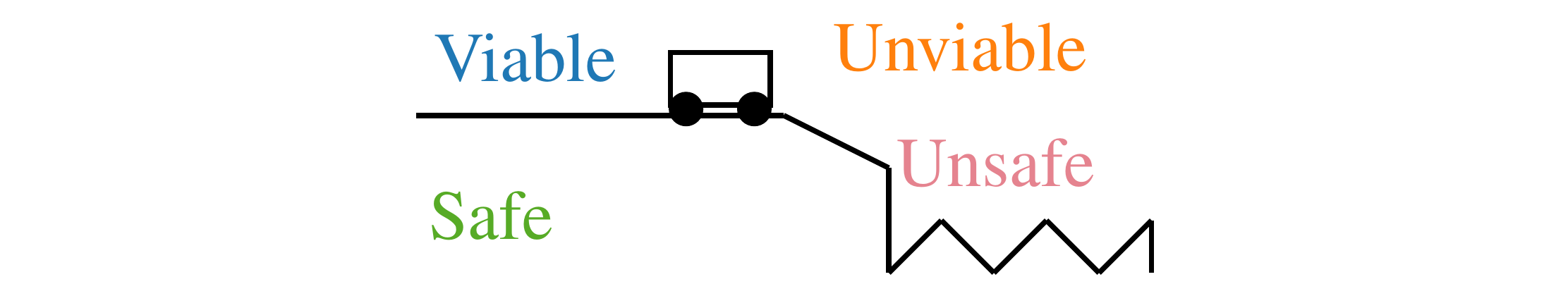}\else\tikzset{
pattern size/.store in=\mcSize, 
pattern size = 5pt,
pattern thickness/.store in=\mcThickness, 
pattern thickness = 0.3pt,
pattern radius/.store in=\mcRadius, 
pattern radius = 1pt}
\makeatletter
\pgfutil@ifundefined{pgf@pattern@name@_tsx3guzus}{
\pgfdeclarepatternformonly[\mcThickness,\mcSize]{_tsx3guzus}
{\pgfqpoint{0pt}{0pt}}
{\pgfpoint{\mcSize+\mcThickness}{\mcSize+\mcThickness}}
{\pgfpoint{\mcSize}{\mcSize}}
{
\pgfsetcolor{\tikz@pattern@color}
\pgfsetlinewidth{\mcThickness}
\pgfpathmoveto{\pgfqpoint{0pt}{0pt}}
\pgfpathlineto{\pgfpoint{\mcSize+\mcThickness}{\mcSize+\mcThickness}}
\pgfusepath{stroke}
}}
\makeatother

\tikzset{
pattern size/.store in=\mcSize, 
pattern size = 5pt,
pattern thickness/.store in=\mcThickness, 
pattern thickness = 0.3pt,
pattern radius/.store in=\mcRadius, 
pattern radius = 1pt}
\makeatletter
\pgfutil@ifundefined{pgf@pattern@name@_jnlnrsv3z}{
\pgfdeclarepatternformonly[\mcThickness,\mcSize]{_jnlnrsv3z}
{\pgfqpoint{0pt}{0pt}}
{\pgfpoint{\mcSize+\mcThickness}{\mcSize+\mcThickness}}
{\pgfpoint{\mcSize}{\mcSize}}
{
\pgfsetcolor{\tikz@pattern@color}
\pgfsetlinewidth{\mcThickness}
\pgfpathmoveto{\pgfqpoint{0pt}{0pt}}
\pgfpathlineto{\pgfpoint{\mcSize+\mcThickness}{\mcSize+\mcThickness}}
\pgfusepath{stroke}
}}
\makeatother

\tikzset{
pattern size/.store in=\mcSize, 
pattern size = 5pt,
pattern thickness/.store in=\mcThickness, 
pattern thickness = 0.3pt,
pattern radius/.store in=\mcRadius, 
pattern radius = 1pt}
\makeatletter
\pgfutil@ifundefined{pgf@pattern@name@_1djbo2pi7}{
\pgfdeclarepatternformonly[\mcThickness,\mcSize]{_1djbo2pi7}
{\pgfqpoint{0pt}{0pt}}
{\pgfpoint{\mcSize+\mcThickness}{\mcSize+\mcThickness}}
{\pgfpoint{\mcSize}{\mcSize}}
{
\pgfsetcolor{\tikz@pattern@color}
\pgfsetlinewidth{\mcThickness}
\pgfpathmoveto{\pgfqpoint{0pt}{0pt}}
\pgfpathlineto{\pgfpoint{\mcSize+\mcThickness}{\mcSize+\mcThickness}}
\pgfusepath{stroke}
}}
\makeatother

\tikzset{
pattern size/.store in=\mcSize, 
pattern size = 5pt,
pattern thickness/.store in=\mcThickness, 
pattern thickness = 0.3pt,
pattern radius/.store in=\mcRadius, 
pattern radius = 1pt}
\makeatletter
\pgfutil@ifundefined{pgf@pattern@name@_8lil2sq05}{
\pgfdeclarepatternformonly[\mcThickness,\mcSize]{_8lil2sq05}
{\pgfqpoint{0pt}{0pt}}
{\pgfpoint{\mcSize+\mcThickness}{\mcSize+\mcThickness}}
{\pgfpoint{\mcSize}{\mcSize}}
{
\pgfsetcolor{\tikz@pattern@color}
\pgfsetlinewidth{\mcThickness}
\pgfpathmoveto{\pgfqpoint{0pt}{0pt}}
\pgfpathlineto{\pgfpoint{\mcSize+\mcThickness}{\mcSize+\mcThickness}}
\pgfusepath{stroke}
}}
\makeatother
\tikzset{every picture/.style={line width=0.75pt}} %
\resizebox{0.5\linewidth}{!}{
\begin{tikzpicture}[x=0.75pt,y=0.75pt,yscale=-1,xscale=1]

\draw  [draw opacity=0][pattern=_tsx3guzus,pattern size=11.25pt,pattern thickness=0.75pt,pattern radius=0pt, pattern color={rgb, 255:red, 31; green, 119; blue, 180}] (3,28) -- (73,28) -- (73,8) -- (3,8) -- cycle ;
\draw  [draw opacity=0][pattern=_jnlnrsv3z,pattern size=11.25pt,pattern thickness=0.75pt,pattern radius=0pt, pattern color={rgb, 255:red, 214; green, 39; blue, 40}] (93,58) -- (143,58) -- (143,28) -- (93,28) -- cycle ;
\draw    (3,28) -- (73,28) ;
\draw  [draw opacity=0][pattern=_1djbo2pi7,pattern size=11.25pt,pattern thickness=0.75pt,pattern radius=0pt, pattern color={rgb, 255:red, 44; green, 160; blue, 44}] (3,58) -- (93,58) -- (93,28) -- (3,28) -- cycle ;
\draw  [fill={rgb, 255:red, 255; green, 255; blue, 255 }  ,fill opacity=1 ] (51.42,16) -- (70.42,16) -- (70.42,26) -- (51.42,26) -- cycle ;
\draw  [fill={rgb, 255:red, 0; green, 0; blue, 0 }  ,fill opacity=1 ] (51.63,26.78) .. controls (51.63,25.25) and (52.88,24) .. (54.42,24) .. controls (55.95,24) and (57.2,25.25) .. (57.2,26.78) .. controls (57.2,28.32) and (55.95,29.57) .. (54.42,29.57) .. controls (52.88,29.57) and (51.63,28.32) .. (51.63,26.78) -- cycle ;
\draw  [fill={rgb, 255:red, 0; green, 0; blue, 0 }  ,fill opacity=1 ] (64.63,26.78) .. controls (64.63,25.25) and (65.88,24) .. (67.42,24) .. controls (68.95,24) and (70.2,25.25) .. (70.2,26.78) .. controls (70.2,28.32) and (68.95,29.57) .. (67.42,29.57) .. controls (65.88,29.57) and (64.63,28.32) .. (64.63,26.78) -- cycle ;

\draw    (73,28) -- (93,38) ;
\draw    (93,58) -- (103,48) ;
\draw    (113,58) -- (103,48) ;
\draw    (93,38) -- (93,58) ;
\draw    (113,58) -- (123,48) ;
\draw    (133,58) -- (123,48) ;
\draw  [draw opacity=0][pattern=_8lil2sq05,pattern size=11.25pt,pattern thickness=0.75pt,pattern radius=0pt, pattern color={rgb, 255:red, 255; green, 127; blue, 14}] (73,28) -- (143,28) -- (143,8) -- (73,8) -- cycle ;
\draw    (143,48) -- (133,58) ;
\draw    (143,48) -- (143,58) ;

\draw (4,41) node [anchor=north west][inner sep=0.75pt]   [align=left] {\textcolor[rgb]{0.34,0.67,0.15}{Safe}};
\draw (93,31) node [anchor=north west][inner sep=0.75pt]  [color={rgb, 255:red, 0; green, 0; blue, 0 }  ,opacity=0.5 ] [align=left] {\textcolor[rgb]{0.8,0.03,0.12}{Unsafe}};
\draw (5,11) node [anchor=north west][inner sep=0.75pt]   [align=left] {\textcolor[rgb]{0.12,0.47,0.71}{Viable}};
\draw (81,9) node [anchor=north west][inner sep=0.75pt]   [align=left] {\textcolor[rgb]{1,0.5,0.05}{Unviable}};

\end{tikzpicture}
} \fi
\caption{Safety vs.\ viability terminology illustrated by a car on a sloped track. \emph{From an engineering perspective, unsafe states (red) are defined as the undesired states in which the car has crashed; their complement is the set of safe states (green). On the slope, insufficient torque prevents the car from returning, so regardless of the control sequence applied, the car inevitably enters the unsafe set in finite time. These states and the unsafe states are defined as unviable (orange). Their complement, i.e., states from which a control sequence exists that can keep the system safe, are viable (blue).
}}
    \label{fig:viablity_vs_safety}
\end{figure}

\newpage
\section*{NeurIPS Paper Checklist}

\begin{enumerate}

\item {\bf Claims}
    \item[] Question: Do the main claims made in the abstract and introduction accurately reflect the paper's contributions and scope?
    \item[] Answer: \answerYes{} %
    \item[] Justification: The abstract and introduction accurately reflect the core contributions: Dyna-SAuR, a safe RL algorithm requiring only initial environment data and a failure criterion, achieving significant reductions in safety violations during training, empirically validated on safety-constrained variants of CartPole and MuJoCo Walker.

\item {\bf Limitations}
    \item[] Question: Does the paper discuss the limitations of the work performed by the authors?
    \item[] Answer: \answerYes{} %
    \item[] Justification: Limitations are discussed in Section~\ref{sec:ConclusionAndLimitations}.

\item {\bf Theory assumptions and proofs}
    \item[] Question: For each theoretical result, does the paper provide the full set of assumptions and a complete (and correct) proof?
    \item[] Answer: \answerYes{} %
    \item[] Justification: Complete proofs of all theoretical results are provided in Appendix~\ref{sec:proofsAppendix}.

    \item {\bf Experimental result reproducibility}
    \item[] Question: Does the paper fully disclose all the information needed to reproduce the main experimental results of the paper to the extent that it affects the main claims and/or conclusions of the paper (regardless of whether the code and data are provided or not)?
    \item[] Answer: \answerYes{} %
    \item[] Justification: The code and a README with full reproduction instructions for each experiment in the submission are provided at \url{https://github.com/Data-Science-in-Mechanical-Engineering/Dyna-SAuR}. Pseudocode and additional algorithmic details are provided in Appendix~\ref{app:pseudocode}.

\item {\bf Open access to data and code}
    \item[] Question: Does the paper provide open access to the data and code, with sufficient instructions to faithfully reproduce the main experimental results, as described in supplemental material?
    \item[] Answer: \answerYes{} %
    \item[] Justification: 
    The code and a README with full reproduction instructions for each experiment in the submission are provided at \url{https://github.com/Data-Science-in-Mechanical-Engineering/Dyna-SAuR}. Pseudocode and additional algorithmic details are provided in Appendix~\ref{app:pseudocode}.

\item {\bf Experimental setting/details}
    \item[] Question: Does the paper specify all the training and test details (e.g., data splits, hyperparameters, how they were chosen, type of optimizer) necessary to understand the results?
    \item[] Answer: \answerYes{} %
    \item[] Justification: Training details and selected hyperparameters are described in Appendix~\ref{app:experimental_setup}.

\item {\bf Experiment statistical significance}
    \item[] Question: Does the paper report error bars suitably and correctly defined or other appropriate information about the statistical significance of the experiments?
    \item[] Answer: \answerYes{} %
    \item[] Justification: All experiments are repeated across ten independent random seeds. Shaded regions in all figures denote 99\% confidence intervals computed via the normal approximation across seeds.

\item {\bf Experiments compute resources}
    \item[] Question: For each experiment, does the paper provide sufficient information on the computer resources (type of compute workers, memory, time of execution) needed to reproduce the experiments?
    \item[] Answer: \answerYes{} %
    \item[] Justification: Full details of the experimental computational resources are reported in Appendix~\ref{app:Computation Resources}.
    
\item {\bf Code of ethics}
    \item[] Question: Does the research conducted in the paper conform, in every respect, with the NeurIPS Code of Ethics \url{https://neurips.cc/public/EthicsGuidelines}?
    \item[] Answer: \answerYes{} %
    \item[] Justification: Dyna-SAuR addresses training-time safety in reinforcement learning, a foundational research problem with potential positive societal implications for deploying reliable autonomous systems. All experimental code is released openly, supporting reproducibility and transparency in accordance with the NeurIPS Code of Ethics.

\item {\bf Broader impacts}
    \item[] Question: Does the paper discuss both potential positive societal impacts and negative societal impacts of the work performed?
    \item[] Answer: \answerNo{} %
    \item[] Justification: Addressing exploration safety, \dynasaur{} contributes to the general applicability of RL methods. There are many potential societal consequences of our work, none of which we feel must be specifically highlighted in the paper.
    
\item {\bf Safeguards}
    \item[] Question: Does the paper describe safeguards that have been put in place for responsible release of data or models that have a high risk for misuse (e.g., pre-trained language models, image generators, or scraped datasets)?
    \item[] Answer: \answerNA{} %
    \item[] Justification: We are not aware of any such risks.

\item {\bf Licenses for existing assets}
    \item[] Question: Are the creators or original owners of assets (e.g., code, data, models), used in the paper, properly credited and are the license and terms of use explicitly mentioned and properly respected?
    \item[] Answer: \answerYes{} %
    \item[] Justification: Original authors of third-party code used in this work are credited, and applicable licensing terms are stated, in Appendix~\ref{app:Computation Resources}.

\item {\bf New assets}
    \item[] Question: Are new assets introduced in the paper well documented and is the documentation provided alongside the assets?
    \item[] Answer: \answerYes{} %
    \item[] Justification: The codebase is thoroughly documented, with newly contributed components explicitly described.

\item {\bf Crowdsourcing and research with human subjects}
    \item[] Question: For crowdsourcing experiments and research with human subjects, does the paper include the full text of instructions given to participants and screenshots, if applicable, as well as details about compensation (if any)? 
    \item[] Answer: \answerNA{} %
    \item[] Justification: The paper does not involve crowdsourcing nor research with human subjects.

\item {\bf Institutional review board (IRB) approvals or equivalent for research with human subjects}
    \item[] Question: Does the paper describe potential risks incurred by study participants, whether such risks were disclosed to the subjects, and whether Institutional Review Board (IRB) approvals (or an equivalent approval/review based on the requirements of your country or institution) were obtained?
    \item[] Answer: \answerNA{} %
    \item[] Justification: The paper does not involve crowdsourcing nor research with human subjects.

\item {\bf Declaration of LLM usage}
    \item[] Question: Does the paper describe the usage of LLMs if it is an important, original, or non-standard component of the core methods in this research? Note that if the LLM is used only for writing, editing, or formatting purposes and does \emph{not} impact the core methodology, scientific rigor, or originality of the research, declaration is not required.
    \item[] Answer: \answerNA{} %
    \item[] Justification: LLMs play no role in the core methodological development of this work. Their use is strictly limited to spelling correction and text refinement of this work.

\end{enumerate}

\end{document}